%% file: main.tex
\documentclass{article}
\usepackage{fullpage}

\usepackage{microtype}
\usepackage{graphicx}
\usepackage{booktabs} %

\usepackage[numbers]{natbib}

\input{math_commands.tex}

\usepackage[colorlinks=true,citecolor=blue]{hyperref}
\usepackage{url}

\usepackage[T1]{fontenc}
\usepackage[utf8]{inputenc}
\usepackage{amsmath}
\usepackage{amsthm}
\usepackage{amssymb}
\usepackage{graphicx}
\usepackage{algorithm}
\usepackage{algpseudocode}
\usepackage{mathtools}
\usepackage{subcaption}
\usepackage{nicefrac}
\usepackage{paralist}
\usepackage{wrapfig}
\usepackage{sidecap}
\usepackage{verbatim}

\makeatletter

\theoremstyle{definition}
\newtheorem{definition}{\protect\definitionname}

\theoremstyle{plain}
\newtheorem{thm}{\protect\theoremname}
\newtheorem{proposition}{Proposition}
\theoremstyle{plain}
\newtheorem{lem}{\protect\lemmaname}
\theoremstyle{plain}

\theoremstyle{remark}

\ifx\proof\undefined\
  \newenvironment{proof}[1][\proofname]{\par
    \normalfont\topsep6\p@\@plus6\p@\relax
    \trivlist
    \itemindent\parindent
    \item[\hskip\labelsep
          \scshape
      #1]\ignorespaces
  }{%
    \endtrivlist\@endpefalse
  }
  \providecommand{\proofname}{Proof}
\fi

\usepackage[capitalize]{cleveref}
\crefname{thm}{Theorem}{Theorems}
\crefname{proposition}{Proposition}{Propositions}
\crefname{lem}{Lemma}{Lemmata}
\crefname{definition}{Definition}{Definitions}
\crefname{cor}{Corollary}{Corollaries}
\crefname{assumption}{Assumption}{Assumptions}

\makeatother

\usepackage{babel}
\providecommand{\definitionname}{Definition}
\providecommand{\lemmaname}{Lemma}
\providecommand{\theoremname}{Theorem}

\begin{document}
\global\long\def\bE{\operatorname{\mathbb{E}}}%
\global\long\def\argmin{\operatorname{arg\,min}}%
\global\long\def\bR{\mathbb{R}}%
\global\long\def\t{\top}%
\global\long\def\cD{\mathcal{D}}%
\global\long\def\cX{\mathcal{X}}%
\global\long\def\cY{\mathcal{Y}}%
\global\long\def\cH{\mathcal{H}}%
\global\long\def\Str{S_{\textnormal{train}}}%
\global\long\def\tstr{\tilde{S}_{\textnormal{train}}}%
\global\long\def\cP{\mathcal{P}}%
\global\long\def\unif{\operatorname{Unif}}%
\global\long\def\var{\operatorname{Var}}%
\global\long\def\cG{\mathcal{G}}%
\global\long\def\eps{\epsilon}%
\global\long\def\fR{\mathfrak{R}}%
\global\long\def\ber{\operatorname{Ber}}%
\global\long\def\cF{\mathcal{F}}%
\global\long\def\bP{\mathbb{P}}%
\global\long\def\ep{\hat{\mathbb{P}}_{N}}%
\global\long\def\ag{\operatorname{Agg}}%
\global\long\def\iid{\stackrel{\textnormal{i.i.d.}}{\sim}}%
\global\long\def\logistic{\operatorname{logistic}}%
\global\long\def\parti{\operatorname{part}}%
\global\long\def\tr{\operatorname{tr}}%

\newcommand{\lc}[1]{\textcolor{red}{\bf [LC]: #1}}

\title{Learning from Aggregated Data: \\ Curated Bags versus Random Bags}
\author{Lin Chen\thanks{Google Research. Email: linche@google.com},
Gang Fu\thanks{Google Research. Email: thomasfu@google.com},
Amin Karbasi\thanks{Google Research and Yale University. Email: aminkarbasi@google.com},
and Vahab Mirrokni\thanks{Google Research. Email: mirrokni@google.com}}

\date{}

\maketitle

\begin{abstract}
Protecting user privacy is a major concern for many machine learning systems that are deployed at scale and collect from a diverse set of population. One way to address this concern is by collecting and releasing data labels in an aggregated manner so that the information about a single user is potentially combined with others. In this paper, we explore the possibility of training machine learning models with aggregated data labels, rather than individual labels. Specifically, we consider two natural aggregation procedures suggested by practitioners: curated bags where the data points are grouped based on common features and random bags where the data points are grouped randomly in bag of similar sizes. For the curated bag setting and for a broad range of loss functions, we show that we can perform gradient-based learning without any degradation in performance that may result from aggregating data. Our method is based on the observation that the sum of the gradients of the loss function on individual data examples in a curated bag can be computed from the aggregate label without the need for individual labels. For the random bag setting, we provide a generalization risk bound based on the Rademacher complexity of the hypothesis class and show how empirical risk minimization can be regularized to achieve the smallest risk bound. In fact, in the random bag setting, there is a trade-off between size of the bag and the achievable error rate as our bound indicates. Finally, we conduct a careful empirical study to confirm our theoretical findings. In particular, our results suggest that aggregate learning can be an effective method for preserving user privacy while maintaining model accuracy.

\end{abstract}

\section{Introduction}
The use of machine learning methods to personalize online services has brought clear benefits for both users and providers, but has also raised concerns about privacy~\citep{papernot2016towards,al2019privacy,de2020overview}. A recent  proposal to address such privacy concerns is to use aggregated data, rather than individual data, to train models~\citep{criteo}. For example, the StoreKit Ad Network (SKAdNetwork) API from Apple aims to measure ad performance metrics such as impressions, clicks, and app installations at an aggregated level, allowing ad networks and advertisers to prioritize privacy concerns~\citep{SKAdNetwork}. The Private Aggregation API of Chrome Privacy Sandbox may also collect user-generated data consisting of instance-label pairs and then enhances anonymity by providing apps and services with bags of instances that are labeled in an aggregated manner~\citep{chrome}. In the context of classification, for example, the proportion of each class among the instances in a bag can serve as an aggregate label, which can be then perturbed appropriately to ensure differential privacy. This is illustrated in \cref{fig:aggregate_labels}.

\begin{SCfigure}
    \centering
    \includegraphics[width=0.55\linewidth]{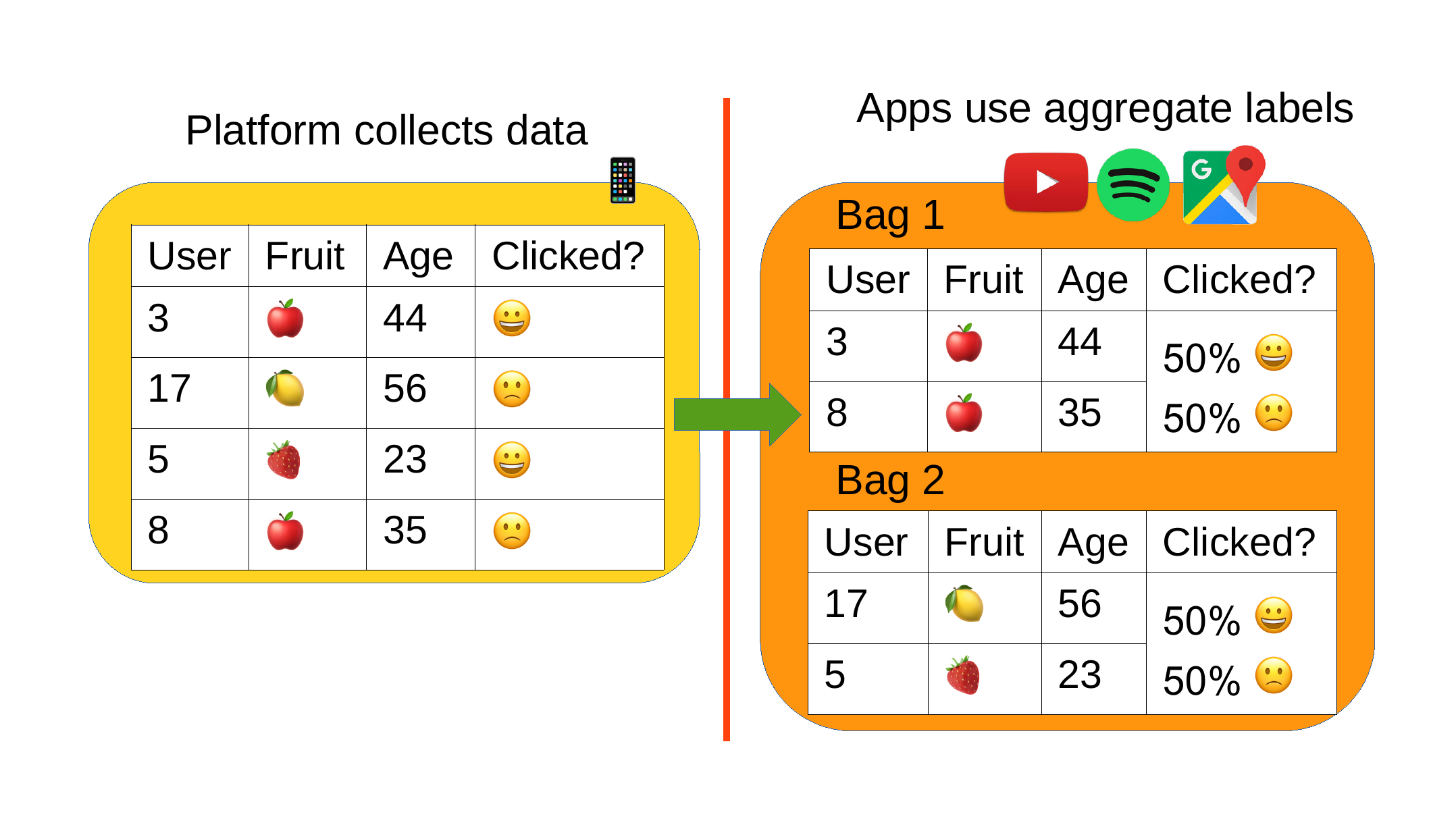}
    \caption{The platform collects user information, such as their \emph{favorite fruit} and \emph{age}, along with a label indicating whether or not they clicked an ad. To protect user privacy, the raw labels are not visible to the apps or services using the data to train machine learning models. Instead, the platform groups the data into bags and provides aggregate labels at the bag level, such as the proportion of users in the bag who clicked the advertisement.
    }
    \label{fig:aggregate_labels}
\end{SCfigure}

In this paper, we explore two recently proposed methods for generating aggregated data labels: \emph{curated bags} and \emph{random bags}. 
 The first method, \emph{curated bags}, was considered in the Criteo Privacy Preserving ML Competition of AdKDD 2021~\citep{diemert2022lessons} and is also implemented in the Chrome Privacy Sandbox. It is primarily designed for datasets with categorical features\footnote{We should highlight that it is not necessary for all feature columns to be categorical.}, but can also be applied to datasets with numerical features by bucketizing them. The process involves selecting a subset of categorical feature columns, aggregating examples that have the same combination of values for those columns into a bag, and labeling each bag with an aggregate label. 
 However, in practice, some bags may be small and could pose a threat to privacy. To address this issue, the small bags can either be filtered out or an appropriate amount of noise can be added to ensure  privacy.
The second method, \emph{random bags}, involves subsampling a predefined number of data points from the training dataset, aggregating the sampled examples into a bag, and labeling the bag with an aggregate label, which is a summary of the individual labels. This makes the individual labels invisible to the apps or services that are using the data. In this paper, we investigate the possibility of learning from aggregated data. Our contributions can be summarized as follows.

\begin{compactitem}

\item For a broad class of loss functions, so called semilinear loss (e.g., mean squared error, log loss, and Poisson loss), we demonstrate that if we use curated bags and if the model is a generalized additive model (whose sub-models are even allowed to share parameters), then we can perform gradient descent-based learning from aggregate labels without any performance loss.

    \item 
    We inverstigate the PAC learnability of random bags by using the Rademacher complexity and propose an estimator that minimizes an empirical bag-level risk. Our generalization bound shows the trade-off between the sample complexity and the size of the bag.

\item We conduct an empirical study of our lossless aggregate learning method, which utilizes the curated bags procedure. We train models on both individual and aggregate labels, and find that the curated bags approach is able to effectively learn from aggregate labels without any loss of performance. We also find that the generalized additive model with neural nets as sub-models outperforms the model with linear feature crosses, and that the curated bags approach outperforms the random bags approach. These results suggest that curated bags are more effective at preserving information during aggregation.
    
\end{compactitem}

We discuss the \textbf{societal impact} of this work in \cref{sec:societal}. All proofs are relegated to the appendix.

\section{Related Work}\label{sec:related}

Aggregate labels are commonly used in group testing methods, such as screening for HIV in donated blood products ~\citep{wein1996pooled} and identifying viral epidemics such as COVID-19~\citep{sunjaya2020pooled}.
There are several prior work on learning with label proportions \citep{shi2018learning,shi2019learning,cui2017inverse,xiao2020new,li2018study,quadrianto2008estimating,patrini2014almost,musicant2007supervised}. While \citet{yu2014learning} studied learning with random bags, they presented a distribution-independent VC dimension bound instead of a distribution-dependent Rademacher complexity bound, which is not only tighter but can also be applied to both classification and regression problems. In addition, our work applies a different random sampling procedure, and it also includes the rigorous study of curated bags. \citet{quadrianto2008estimating} examined how to estimate labels from label proportions using a specific generative model. Similarly, \citet{zhang2020learning} applied the maximum likelihood method and developed theoretical guarantees by introducing the concept of consistency up to an equivalence relation. \citet{musicant2007supervised} presented a framework for learning from aggregate outputs and demonstrated adaptations of several classical machine learning algorithms. Other related work includes the proportion-SVM ($\propto$SVM) method~\citep{yu2013proptosvm}, a boosting method for learning with label proportions~\citep{qi2017adaboost}, and extensions of nonparallel SVM that can learn with label proportions~\citep{qi2016learning,chen2017learning}. Recently, \citet{saket2022combining} studied the problem of combining bag distributions to better learn from label proportions.
Another related area of research is label differential privacy~\citep{chaudhuri2011sample,beimel2013private,wang2019sparse,ghazi2021deep}. In this setting, the labels of individual instances are considered sensitive and require protection, while their features are considered non-sensitive. Learning from aggregate labels can be seen as an approach to achieving label privacy.

\section{Lossless Learning from Aggregate Labels Under Curated Bags}\label{sec:lossless}

We use the shorthand notation $n$ to denote the set $\{1, 2, \dots, n\}$. We denote the data domain by $\cX$, and the label domain by $\cY$.
We begin with the aggregating strategy called \emph{curated bags}, a way of grouping examples in a dataset by their feature value combination. This aggregation method creates a partition of the entire dataset, where all examples in the same bag share the same feature value combination on the selected feature columns, while examples in different bags differ on those. In what follows, we show that  the curated bags aggregation can achieve the same performance as learning from individual labels in  the classical machine learning setting. To do so, we need to introduce two key components:

\begin{compactitem}
\item A \textbf{semilinear loss function} $\ell(y,\hat{y})$ is a loss function that is composed of  a linear  and a nonlinear function of the model prediction $\hat{y}$. Some widely used loss functions, such as the mean squared error, log loss, and Poisson loss, are all special cases of semilinear loss functions.
\item A \textbf{generalized additive model} (GAM) is a statsitical learning model that can be decomposed into a sum of several sub-models, allowing for different  model capacity and expressivity of each sub-model. Each sub-model can be a neural network, decision tree, etc. The sub-models are allowed to share parameters, meaning that their parameter sets do not have to be non-overlapping. This allows GAMs to be more flexible than traditional regression models, which can only model a single relationship between a response variable and a set of predictors.%
     
\end{compactitem}

\subsection{Feature-based Curated Bags}

\begin{figure}[t]
    \centering
    \includegraphics[width=\linewidth]{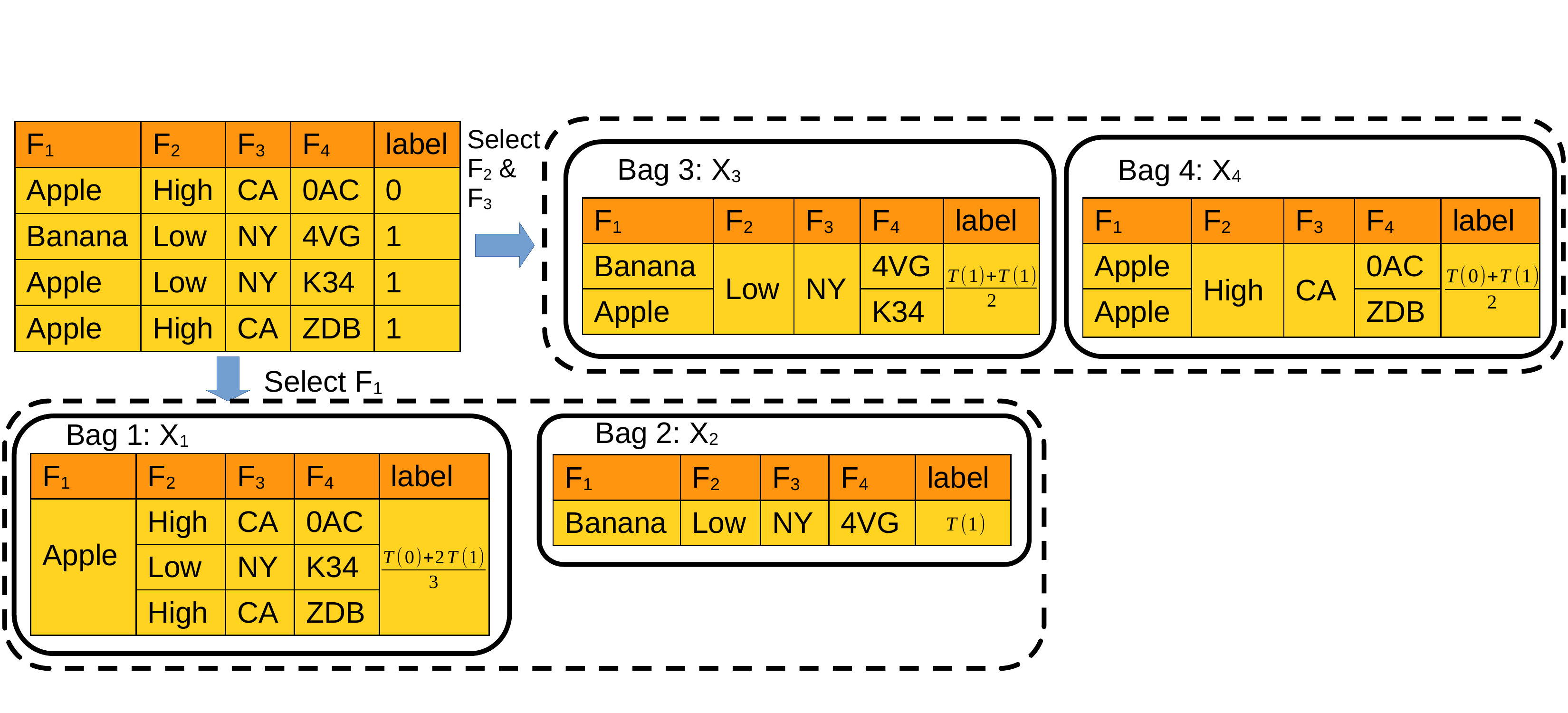}
    \caption{The figure illustrates the construction of curated bags, in which the data is partitioned according to the feature values. The original data, at top left, contains 4 features and 1 binary label. The data is first partitioned into two bags, based on the first feature. The second and third features are then considered, and the data is partitioned into two more bags, based on all possible combinations of the two features. The aggregate label for each bag is the average of the transformed label values.
    \label{fig:curated_bags}}

\end{figure}

In this section, we consider an aggregate label generation procedure termed \emph{feature-based curated bags}, or curated bags for short. This procedure is inspired by the Criteo Privacy Preserving ML Competition of AdKDD 2021~\citep{diemert2022lessons}, and it assumes that all features are categorical. For non-categorical features, e.g., numerical features, one may categorize them and transform them into categorical features and apply the curated bag aggregation procedure.

Specifically, curated bags are formed by partitioning the training dataset according to the feature values of examples. In \cref{fig:curated_bags}, we illustrate an example. The table in the top left is the original dataset, which consists of 4 feature columns $F_1,F_2,F_3,F_4$ and 1 binary label column $y$. We first choose the feature value $F_1$ to partition the dataset into two bags $X_1$ and $X_2$ (the two tables in the dash line box in the left bottom). In $X_1$, the value of $F_1$ for all examples is \emph{apple}. The value of $F_1$ for the data example in $X_2$ is \emph{banana}. The aggregate label is the average of the transformed labels. \textbf{The transform function $T(\cdot)$ will be chosen according to the loss function in \cref{def:aggregatable_loss} and \eqref{eq:semilinear}.} For example, the original labels of examples in $X_1$ are $0,1,1$. Using some transformation function, we get the transformed labels $T(0),T(1),T(1)$. The aggregate label is the average of the transformed labels, which is $\frac{T(0)+2T(1)}{3}$.
We can also select a combination of more than one feature. By selecting the combination of features $F_2$ and $F_3$, we partition the dataset based on the values of those features. There are two possible combinations: (Low, NY) and (High, CA). The second and third data examples in the table have the combination (Low, NY), while the first and last examples have the combination (High, CA). As a result, we group the second and third examples together in $X_3$, and the first and last examples together in $X_4$. The aggregate label for each bag is the average of the $T(\cdot)$ value of the original labels in that bag.

This bagging approach produces more informative aggregate labels than random bags because it partitions the training dataset based on features, which ensures that each bag contains examples that are more likely to be relevant to each other. This in turn makes it more likely that the aggregate label for each bag will be informative and useful for training a supervised learning model.

We now formally define curated bags. Let $\Str$ be the raw training dataset of $N$ examples, defined as
$\Str = \left\{(x^{(i)}, y^{(i)}) \mid i \in [N] \right\}$
where each example $x^{(i)} \in V_1 \times \cdots \times V_d$ is a vector of $d$ features. Each feature $F_i$ is associated with a finite set of possible values $V_i$, which we call the vocabulary of $F_i$. For each example $x\in \Str$, we denote the value of the $i$-th feature by $F_i(x)$.

Given a subset of features $C$, Algorithm~\ref{alg:bag-generation} shows how to generate curated bags and aggregate labels by partitioning the training dataset based on the combination of feature values of features $C$. The examples with the same feature values are grouped into a bag. The aggregate label for each bag is the average of the original labels after applying a transform function $T(\cdot)$.

\newcommand{\CB}{\mathsf{CuratedBags}}
\newcommand{\MCB}{\mathsf{MultiCuratedBags}}

\begin{algorithm}[htb]
\caption{$\CB(C)$: Generate curated bags by partitioning the training dataset $\Str$ by feature set $C$ \label{alg:bag-generation}}
\begin{algorithmic}[1]
\Require Selected features $C = \{c_1,c_2,\dots,c_{|C|}\}\subseteq [p]$.
    \ForAll{ $(v_1,v_2,\dots,v_{|C|})\in V_1\times V_2\times \cdots V_{|C|}$}
    \State  $\Str^{(v_1,v_2,\dots,v_{|C|})} \gets \{(x,y)\in \Str \mid F_i(x) = v_i,\forall i\in C\}$
    \State Generate a bag $X_{(v_1,v_2,\dots,v_{|C|})} \gets \{ x\mid (x,y)\in \Str^{(v_1,v_2,\dots,v_{|C|})} \}$.
    \State Generate the aggregate label $\bar{y}_{(v_1,v_2,\dots,v_{|C|})}\gets \frac{1}{|X_{(v_1,v_2,\dots,v_{|C|})}|} \sum_{(x,y)\in \Str^{(v_1,v_2,\dots,v_{|C|})} } T(y)$.\label{ln:T}
    \EndFor
\State \Return $\left\{\left(X_{(v_1,v_2,\dots,v_{|C|})}, \bar{y}_{(v_1,v_2,\dots,v_{|C|})}\right)\mid (v_1,v_2,\dots,v_{|C|})\in V_1\times V_2\times \cdots V_{|C|} \right\}$
\end{algorithmic}
\end{algorithm}

Usually, we choose more than one subset of features to partition the training dataset and generate curated bags. We denote the set of selected feature sets by $\mathcal{C} = \{C_1, C_2, \dots, C_{|\mathcal{C}|}\}$. For each selected feature set, we use Algorithm~\ref{alg:bag-generation} to generate curated bags and aggregate labels.

\begin{algorithm}[htb]
\caption{$\MCB(\mathcal{C})$: Generate curated bags by partitioning the training dataset $\Str$ by multiple feature sets $\mathcal{C}$ \label{alg:bag-generation2}}
\begin{algorithmic}[1]
\Require $\mathcal{C}=\{C_1, C_2, \dots, C_{|\mathcal{C}|}\}$ where each $C_i\subseteq[p]$
\For{$C\in \mathcal{C}$}
    \State Generate curated bags and aggregate labels using $\CB(C)$ in \cref{alg:bag-generation}
\EndFor
\end{algorithmic}
\end{algorithm}

\subsection{Loss Function and Generalized Additive Model}

\paragraph{Combinable function.}

In the following, we introduce the notion of an combinable function with respect to a model class and a bag (with its aggregate label). The idea of an combinable function is that the sum of this function over a bag of examples and labels can be computed through only the examples and their aggregate label.

\begin{definition}[Combinable function]
Let $X = \{x_j\}_{j \in [m]} \subseteq \mathcal{X}$ be a bag of examples whose individual labels are $\{y_j\}_{j\in [m]}$, and $\bar{y} = \phi(y_1, \dots, y_m)$ be the aggregate label of $X$. A function $g:\cX\times \cY\to \bR$ is combinable with respect to $(X,\bar{y})$ if there exists a function $J$ such that
\begin{equation}
\frac{1}{m}\sum_{j\in [m]} g(x_j,y_j) = J(x_1,\dots,x_m,\bar{y}).
\end{equation}
\end{definition}

Let $X = \{x_j\}_{j \in [m]} \subseteq \mathcal{X}$ be a bag of examples whose individual labels are $\{y_j\}_{j\in [m]}\subseteq \cY$, and let $\bar{y} = \phi(y_1, \dots, y_m)$ be the aggregate label of $X$. Given the loss function $\ell:\cY\times \cY \to \bR$, and the model class $\mathcal{H} = \{f_\beta \colon \mathcal{X} \to \mathbb{R}^K \mid \beta \in \Theta\subseteq \bR^p \}$, when we perform differentiable learning, we need to compute the derivative 

$$\frac{1}{m}\sum_{j\in [m]}\frac{\partial \ell(y_j,f_\beta(x_j))}{\partial \beta_i}.$$

Note that if the function $\frac{\partial \ell(y,f_\beta(x))}{\partial \beta_i}$ turns out to be combinable, then we suffer no loss from aggregate learning because we can recover the derivative using the aggregate label, as in the usual supervised learning with data examples and individual labels. In the following, we introduce a large class of loss functions, called semi-linear losses, for which we can prove that they are combinable (see Theorem~\ref{thm:lossless-general}).

\begin{definition}[Semilinear loss]
\label{def:aggregatable_loss}
A loss function $\ell(y, \hat{y})$ of the label $y\in \bR^K$ and the predicted value $\hat{y}\in \bR^K$ is said to be semilinear if it can be written in the following form:
\begin{equation}\label{eq:semilinear}
\ell(y, \hat{y}) = b(\hat{y}) - T(y)^\top  \hat{y} + c(y),
\end{equation}
for some functions $b(\cdot)\in \bR, T(\cdot)\in \bR^K, c(\cdot)\in \bR$, where $T(\cdot)$ is the transform function of the label.
\end{definition}

The family of semilinear loss functions encapsulates a variety of loss functions in machine learning, which includes mean squared error (for regression), log loss (for classification) and Poisson loss as special cases. 

\def\LSE{\operatorname{LSE}}

\begin{compactitem}
    \item \textbf{Mean squared error.} If we set $b(\hat{y}) = \frac{1}{2} \|\hat{y}\|^2$, $c(y) = \frac{1}{2} \|y\|^2$ and $T(y)= y$, then we have $\ell(y,\hat{y}) = \frac{1}{2} \|\hat{y}\|^2 - y^\top \hat{y} + \frac{1}{2} \|y\|^2 = \frac{1}{2} \|y-\hat{y}\|^2$, which is the mean squared error. 
    \item \textbf{Log loss and cross-entropy loss.}  For log loss and cross entropy loss, we require that the label $y$ be a one-hot vector. If we set $b(\hat{y}) = \LSE(\hat{y}) \triangleq \log(\sum_{i\in [K]} e^{\hat{y}_i})  $ ($\LSE$ is known as the LogSumExp function), $c(y) = 0$, and $T(y) = y$, then we have $
        \ell(y,\hat{y}) = -y \hat{y} + \LSE(\hat{y}_i) = \sum_{i\in [K]} 1_{\{y_i=1\}}\left(-\hat{y}_i + \LSE(\hat{y}_i)\right) = - \sum_{i\in [K]} 1_{\{y_i=1\}} \log\frac{e^{\hat{y}_i}}{\sum_{j\in [K] } e^{\hat{y}_j}}
    $.
    \item \textbf{Poisson loss.} If we set $K=1$, $b(\hat{y})=e^{\hat{y}}$, $c(y)=0$, and $T(y)=y$, then we have $\ell(y,\hat{y}) = e^{\hat{y}} - y\hat{y}$.
\end{compactitem}

The above examples show the generality of semilinear loss functions. 

\subsection{Semilinear Losses and Generalized Additive Models Under Curated Bags}

We demonstrate the aggregability of the derivative with respect to parameters when using a semilinear loss and a generalized additive model under curated bags. We begin with the simplest case and make a few assumptions to better convey the intuition. These assumptions will be relaxed later.

\begin{compactitem}
    \item \textbf{Scalar semilinear loss.} We use the semilinear loss with $K=1$ so $y$ and $\hat{y}$ are scalars. Moreover, assume for now that $T(y)=y$ is the identity function (this already covers the mean squared error, log loss and Poisson loss). So in this case, $\ell(y,y') = b(\hat{y})-y \hat{y}+c(y)$.
    \item \textbf{Single-indexed curated bag.} We use a curated bag $X = \{x_i\}_{i \in [m]}$, where $x_i$ is a data example, obtained by partitioning the training dataset by the feature value of the $j_0$-th feature. Therefore, for all $i \in [m]$, the value of the $j_0$-th feature of $x_i$ is equal.
    \item \textbf{Additive model.} Suppose that $f(x;\beta)$ 
    can be written as a sum of several sub-models, each parameterized by $\beta_j$:
    
$$f(x;\beta) = \sum_{j\in [p]} f_j(x_j;\beta_j),$$

where $\beta = \begin{pmatrix}
     \beta_1, 
     \beta_2,
     \dots ,
     \beta_p
\end{pmatrix}^\top $, $x = \begin{pmatrix}
     x_1  ,
     x_2 ,
     \dots ,
     x_p
\end{pmatrix}^\top $.
\end{compactitem}

We summarize the result of lossless aggregate learning in its simplest version in Proposition \ref{prop:simple-lossless}. This proposition implies that the sum of the derivatives of the loss function with respect to the $j_0$-th feature on all individual data examples in a curated bag can be obtained from the aggregate label, if the curated bag is obtained by partitioning the training dataset according to the feature value of the $j_0$-th feature. If we have multiple curated bags, and they are all obtained by partitioning the training dataset according to the feature value of the $j_0$-th feature, the sum of the derivatives of the loss function with respect to the $j_0$-th feature on all data examples in these curated bags can also be obtained from the aggregate labels of these bags. We simply need to sum the right-hand side of Equation \ref{eq:simple-lossless} for each bag.

\begin{proposition}[Lossless Aggregate Learning, Simplest Version]\label{prop:simple-lossless}
Let $X = \{x_i\}_{i \in [m]}$ be a curated bag of examples, and for all $i \in m$, let the value of the $j_0$-th feature of $x_i$ be equal. Let the label of $x_i$ be $y_i$ and the aggregate label be $\bar{y}=\frac1m\sum_{i\in [m]}y_i$. Let the loss function be $\ell(y,\hat{y}) = b(\hat{y})-y \hat{y}+c(y)$ and the model be $\hat{y}_i=f(x_i;\beta) = \sum_{j\in [p]} f_j(x_{i,j};\beta_j)$, where both $x$ and $\beta$ are $p$-dimensional vectors and $x_{i,j}$ is the $j$-th entry of $x_i$.  Then, we have
\begin{equation}\label{eq:simple-lossless}
    \frac{1}{m}\sum_{i\in [m]} \frac{\partial \ell(y_i,\hat{y}_i)}{\partial \beta_{j_0}} = \frac{\partial f_{j_0}(x_{1,j_0};\beta_{j_0})  }{\partial \beta_{j_0}} \left(\frac{1}{m}\sum_{i\in [m]} b'(\hat{y_i}) - \bar{y} \right) \,.
\end{equation}
\end{proposition}

\begin{proof}
Since the proof is simple and provides important intuitions, we provide it in the main body of the paper. Let us calculate the derivative with respect to a parameter entry $\beta_{j_0}$:
$\frac{\partial \ell(y_i,f(x_i;\beta))}{\partial \beta_{j_0}} = \left( b'(\hat{y}_i) - y_i \right) \frac{\partial f(x_i;\beta) }{\partial \beta_{j_0}} = \left( b'(\hat{y}_i) - y_i \right)  \frac{\partial f_{j_0}(x_{j_0};\beta_{j_0})  }{\partial \beta_{j_0}}  $,
where $\hat{y}_i = f(x_i,\beta)$. The last equality is because only the sub-model $f_{j_0}(x_{j_0};\beta_{j_0})$ depends on $\beta_{j_0}$.

Summing over $i \in [m]$, we get

\begin{align*}
\frac{1}{m}\sum_{i\in [m]} \frac{\partial \ell(y_i,\hat{y}_i)}{\partial \beta_{j_0}} &= \frac{1}{m} \sum_{i\in [m]}\left( b'(\hat{y_i}) - y_i \right)  \frac{\partial f_{j_0}(x_{i,j_0};\beta_{j_0})  }{\partial \beta_{j_0}} 
= \frac{1}{m}\frac{\partial f_{j_0}(x_{1,j_0};\beta_{j_0})  }{\partial \beta_{j_0}} \sum_{i\in [m]}\left( b'(\hat{y_i}) - y_i \right)  \\
&= \frac{\partial f_{j_0}(x_{1,j_0};\beta_{j_0})  }{\partial \beta_{j_0}} \left(\frac{1}{m}\sum_{i\in [m]} b'(\hat{y_i}) - \bar{y} \right) 
\end{align*}
where $\bar{y} = \frac{1}{m}\sum_{i\in [m]} y_i$ is the aggregate label and the second equality is because the feature value $x_{i,j}$ of $x_i$'s in this bag are all equal, and therefore $\frac{\partial f_{j_0}(x_{i,j_0};\beta_{j_0})  }{\partial \beta_{j_0}}$ are the same for all $i$, and thus all equal to $\frac{\partial f_{j_0}(x_{1,j_0};\beta_{j_0})  }{\partial \beta_{j_0}}$.
\end{proof}

Although \cref{prop:simple-lossless} discusses how to obtain the derivative of the loss function with respect to a specific feature, in practice, we need to know the gradient of the loss function with respect to all parameters in $\beta$ in order to train them. However, if we partition the training dataset according to the $j$th feature value to obtain the curated bags $\CB(j)$ for each $\beta_j$ ($j\in d$), we can compute the gradient using the aggregate labels.

In the following, we will relax our assumptions and extend the result of lossless aggregate learning in Proposition~\ref{prop:simple-lossless} to a more general setting. We relax the assumptions in several aspects. First, we
consider a more general model in which not only  sub-models can share parameters but also may have more than one parameter (in contrast to the assumptions of  \cref{prop:simple-lossless}, which states that each sub-model has only a distinct parameter). More formally, let $\beta\in \bR^p$ be the parameter of the model. The generalized additive model model has the following form:
\begin{equation}\label{eq:gam}
\hat{y} = f(x;\beta) = \sum_{j\in [n_E]} f_j(x_{E'_j};\beta_{E_j})\in \bR^K\,,
\end{equation}
 where $\hat{y}$ is the model prediction of data example $x$, $n_E$ is the number of sub-models, $E_j\subseteq [p]$ is a set of indices of entries of $\beta$ , $E'_j\subseteq [d]$ is a set of indices of entries of $x$, $\beta_{E_j}$ denotes the sub-vector indexed by $E_j$ and $x_{E'_j}$ denotes the sub-vector indexed by $E'_j$. Second,  to compute the derivative with respect to the $j$-th feature, we can use all curated bags $C_i\in \mathcal{C}=\{C_1, C_2, \dots, C_{|\mathcal{C}|}\}$ with aggregate labels $$\CB(C_i) = \left\{\left(X_{(v_1,v_2,\dots,v_{|C_i|})}, \bar{y}_{(v_1,v_2,\dots,v_{|C_i|})}\right)\mid (v_1,v_2,\dots,v_{|C_i|})\in V_1\times V_2\times \cdots V_{|C_i|}\right\},$$ generated by \cref{alg:bag-generation}. Third, instead of simply considering scalar labels in semilinear losses, we now extend the results to the multidimensional setting.

\begin{thm}[Lossless Aggregate Learning]\label{thm:lossless-general}
Let $\ell$ be the semilinear loss function defined in \eqref{eq:semilinear} and we consider the generalized additive model defined in \eqref{eq:gam}. We assume that every parameter entry is used in the model, i.e., $\bigcup_{j \in [n_E]} E_j = [p]$.
Furthermore, for every $j$, there exists $\phi(j) \in [|\mathcal{C}|]$ such that $E'_j \subseteq C_{\phi(j)}$. Let $X$ be a curated bag in $\CB(C_{\phi(j)})$ and define $E'_j(X) \triangleq x_{E'_j}$ for $x \in X$.
\footnote{Recall that the feature value of the feature columns $C_{\phi(j)}$ is identical for every $x\in X$ due to the construction of curated bags. Therefore, since $E'_j$ is a subset of $C_{\phi(j)}$, the expression $x_{E'_j}$ does not depend on which $x$ is chosen from the bag $X$.} We have   \begin{equation}\label{eq:main-thm}
         \sum_{(x,y)\in \Str} \frac{\partial \ell(y,\hat{y})}{\partial \beta_{j_0}} 
    ={} \sum_{j\in [n_E]:j_0\in E_j} 
    \sum_{(X,\bar{y})\in \CB(C_{\phi(j)})}
   \left(  \frac{\partial f(E'_j(X);\beta_{E_j})}{\partial \beta_{j_0}}\right)^\top \sum_{x\in X}
    (\nabla_{\hat{y}} b(\hat{y}) - \bar{y})\,.
\end{equation}
\end{thm}
The left-hand side of \cref{eq:main-thm} is the derivative of the loss function with respect to the model parameters, evaluated on all examples in the training dataset with individual labels. The right-hand side of the equation only uses the aggregate labels. This means that we can perform gradient-based learning from aggregate labels.
\cref{thm:lossless-general} assumes that for every $j$, there exists $\phi(j) \in [|\mathcal{C}|]$ such that $E'_j \subseteq C_{\phi(j)}$.  This assumption implies that the curated bags used in the sub-model $f_j(x_{E'_j};\beta_{E_j})$ are obtained by partitioning the training dataset according to the feature value combinations of $E'_j$ (in this case, $E'_j = C_{\phi(j)}$) or even a superset of $E'_j$ (in this case, $E'_j \subseteq C_{\phi(j)}$, and we get a finer partition than $E'_j$).

\section{Learnability Under Random Bags}\label{sec:randomly}

This section investigates a multilabel multiclass classification problem of learning from aggregate data. In this problem, the learner is given data bags that are formed by resampling examples from the training dataset without replacement.
We emphasize that data examples are sampled without replacement within each bag.
We follow the same procedure across bags, i.e.,  replacing all examples and sampling without replacement. This means that same examples may appear in multiple bags.

We assume that the data $(x, y)$ is drawn from an unknown distribution $\mathbb{P}$ over $\mathcal{X} \times \mathcal{Y}$, where $x$ is the example and $y$ is the label. We choose $\mathcal{Y} = \{0, 1\}^K$, which generalizes $K$-class classification as a special case. In our general multilabel multiclass classification setting, the label $y$ can have multiple non-zero entries. We define $h(x)[k]$ and $y[k]$ to be the $k$-th entry of $h(x)$ and $y$, respectively. We evaluate the performance of a model $h : \mathcal{X} \to \mathcal{Y}$ by the expected Hamming distance
$$R(h) = \mathbb{E}_{(x, y) \sim \mathbb{P}} \left[ \sum_{k \in [K]} \mathbb{I}(h(x)[k] \neq y[k]) \right],$$
which is the expected number of entries in which $h(x)$ and $y$ disagree.

We denote the marginal distribution of $x$ (and $y$, respectively) under $\bP$ by $\bP\mid_{x}$ (and $\bP\mid_{y}$, respectively). Let $\cH\subseteq\left\{ f:\cX\to\cY\right\} $ be a hypothesis set. Let $\ep=\frac{1}{N}\sum_{i\in[N]}\delta_{\left(x^{(i)},y^{(i)}\right)}$ be an empirical measure of $\bP$, where $\delta_{\left(x^{(i)},y^{(i)}\right)}$ is the Dirac measure at
$\left(x^{(i)},y^{(i)}\right)\iid\bP$. The empirical measure $\ep$ models the uniform distribution on the training dataset $\Str\triangleq\left\{ \left(x^{(i)},y^{(i)}\right)\mid i\in[N]\right\} $. 

Crucially, and in contrast to the classic learning setting, the raw training data is invisible to the learning algorithm. Instead, the learning algorithm has access to $n$ i.i.d. samples $(X_{i},\bar{y}_{i})$ that are obtained from the following process (denote the distribution of $(X_{i},\bar{y}_{i})$ by $\ag(\ep)$):

\begin{compactitem}
\item For each $i\in[n]$, we resample $m$ example-label pairs $S_{i}\triangleq\left\{ \left(x_{i,j},y_{i,j}\right)\mid j\in[m]\right\} $ from the training dataset $\Str$ uniformly at random without replacement;
\item We set $X_{i}=\left\{ x_{i,j}\mid j\in[m]\right\} $ and $\bar{y}_{i}\mid y_{i,1},y_{i,2},\dots,y_{i,m}\sim\ber\left(\frac{1}{m}\sum_{j\in[m]}y_{i,j}\right)$.
\end{compactitem}

We study the problem of learning a hypothesis $f \in \mathcal{H}$ from the samples $\{(X_i, \bar{y}_i)\}_{i \in [n]}$. Since we consider an agnostic probably approximately correct (PAC) learning setup~\citep[Chapter~2]{mohri2018foundations}, we are interested in upper-bounding the excess risk:
$R(\hat{h}) - \inf_{h \in \mathcal{H}} R(h)$
which is the gap between the risk of the hypothesis that our algorithm selects $R(\hat{h})$ and that of the optimal one in the hypothesis class $\inf_{h \in \mathcal{H}} R(h)$. We will bound the excess risk by the Rademacher complexity, defined as follows.

\begin{definition}[Rademacher complexity~\citep{bartlett2002rademacher,mohri2018foundations}]
The \emph{Rademacher complexity} of $\cH\subseteq\left\{ h:\cX\to\bR\right\} $
is defined by
$\fR_{n,P}\left(\cH\right)\triangleq{} \bE
     \sup_{h\in\cH}\frac{1}{n}\sum_{i\in[n]}\sigma_{i}h(x_{i})$, where $\{x_{i}\}_{i\in [n]}$ are i.i.d. with distribution $P$ and $\{\sigma_{i}\}_{i\in [n]}$
are independent Rademacher random variables. 
\end{definition}

For multilabel classification, a hypothesis outputs a $K$-dimensional vector. We introduce the \emph{flattened Rademacher complexity} to measure the correlation between Rademacher random variables and all entries of the hypothesis's output.
\begin{definition}[Flattened Rademacher complexity]
The \emph{flattened Rademacher complexity} of $\cH\subseteq\left\{ h:\cX\to\bR^{K}\right\} $
is defined by $\fR_{n,P}^{+}\left(\cH\right)\triangleq{}\bE
     \sup_{h\in\cH}\frac{1}{n}\sum_{i\in[n],k\in[K]}\sigma_{i,k}h(x_{i})[k]$,
where $\{x_{i}\}_{i\in [n]}$ are i.i.d.\ with distribution $P$ and $\{\sigma_{i,k}\}_{i\in [n],k\in [K]}$
are independent Rademacher random variables. 
\end{definition}
In \cref{thm:rademacher}, we propose a new estimator $\hat{h}$ and upper-bound its excess risk $R(\hat{h}) - \inf_{h \in \mathcal{H}} R(h)$ by a combination of the usual and flattened Rademacher complexities of the hypothesis class.
\begin{thm}
\label{thm:rademacher} 

If $\mathcal{H}_k \triangleq \{ x \mapsto h(x)[k] \mid h \in \mathcal{H} \}$ and  $\hat{h} \in \cH $ is a minimizer of $ \frac{1}{n}\sum_{i \in [n]} \left\| \frac{1}{m}\sum_{x \in X_i} h(x) - \bar{y}_i \right\|_2^2 - \frac{(m-1)N}{m(N-1)}  \left\| \frac{1}{n}\sum_{i \in [n]} \left( \frac{1}{m}\sum_{x \in X_i} h(x) - \bar{y}_i \right) \right\|_2^2$,
then with probability $1-4\delta$, $R(\hat{h}) - \inf_{h \in \mathcal{H}} R(h)$ is upper bounded by:
    \begin{align}
 &{} \frac{8\left(m-1\right)N}{N-m} 
 \left(\sum_{k\in[K]}\fR_{n,\ep\mid_{x}}(\cH\mid_{k})+K\sqrt{\frac{\log\left(2mK/\delta\right)}{2n}}\right)
 +2\left(4\sqrt{2K}\fR_{N,\bP\mid_{x}}^{+}(\cH)+\sqrt{\frac{\log2/\delta}{2N}}\right)\nonumber \\
 & +\frac{2m\left(N-1\right)}{N-m}\left(4\sqrt{2K}\fR_{n,\ep\mid_{x}}(\cH)+\sqrt{\frac{\log\left(2/\delta\right)}{2n}}\right)\,.\label{eq:excess-risk-bound}
    \end{align}
\end{thm}

First, we would like to remark that if every bag only contains a single example (in this case $m=1$ and it reduces to the usual classification with individual labels), the correction term $-\frac{\left(m-1\right)N}{m\left(N-1\right)}\left\Vert \frac{1}{n}\sum_{i\in[n]}\left(\frac{1}{m}\sum_{x\in X_{i}}h(x)-\bar{y}_{i}\right)\right\Vert _{2}^{2}$ becomes zero. Second, there are three terms in \cref{eq:excess-risk-bound}. The first term has a factor of $\frac{2\left(m-1\right)N}{N-m}$ and the third term has a factor of $\frac{2m\left(N-1\right)}{N-m}$. If the bag size approaches the size of the entire training dataset ($m\to N$), both factors go to infinity and thereby drive the first and third terms to infinity, which also agrees with our intuition.

\section{Experiments}\label{sec:experiments}

\begin{figure*}
    \centering
    \begin{subfigure}[t]{0.32\linewidth}
         \centering
         \includegraphics[width=\linewidth]{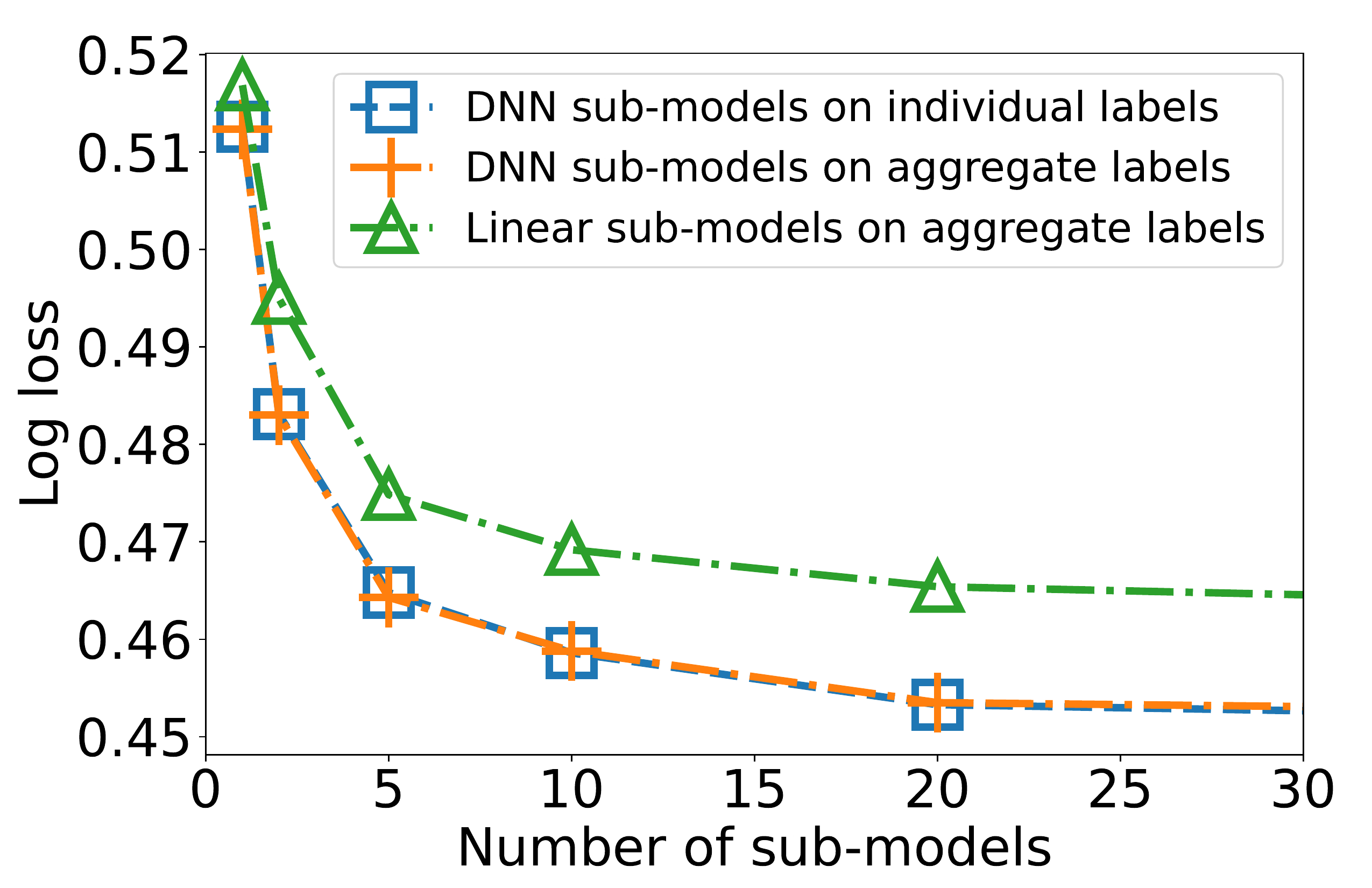}
         \caption{Test error decreases with more sub-models under curated bags without DP noise\label{fig:noise_free}}
     \end{subfigure}~
    \begin{subfigure}[t]{0.32\linewidth}
         \centering
         \includegraphics[width=\linewidth]{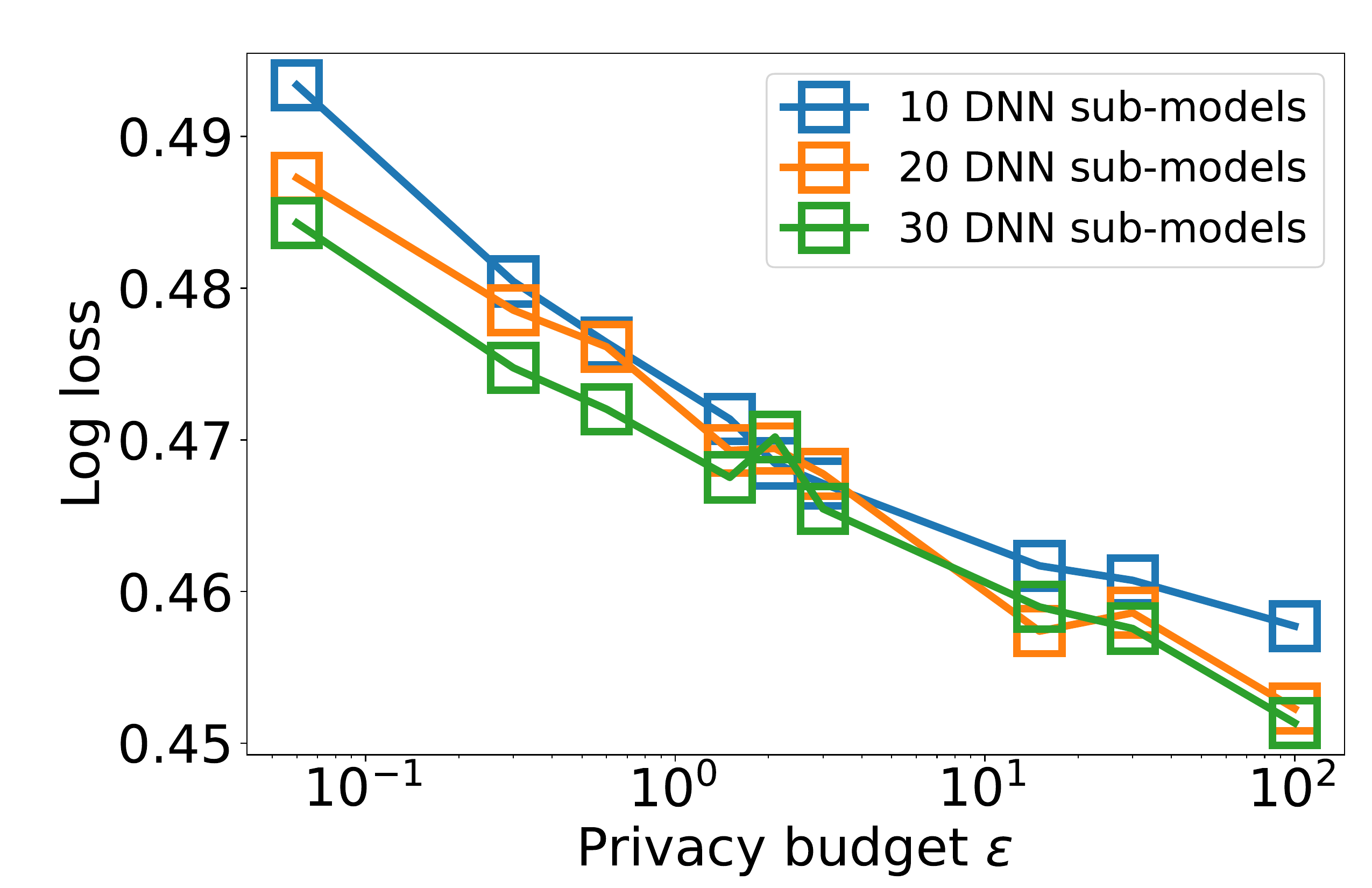}
         \caption{Test error decreases with lower DP level $\eps$ under curated bags\label{fig:sub_models_vs_eps}}
     \end{subfigure}~
    \begin{subfigure}[t]{0.32\linewidth}
         \centering
        \includegraphics[width=\linewidth]{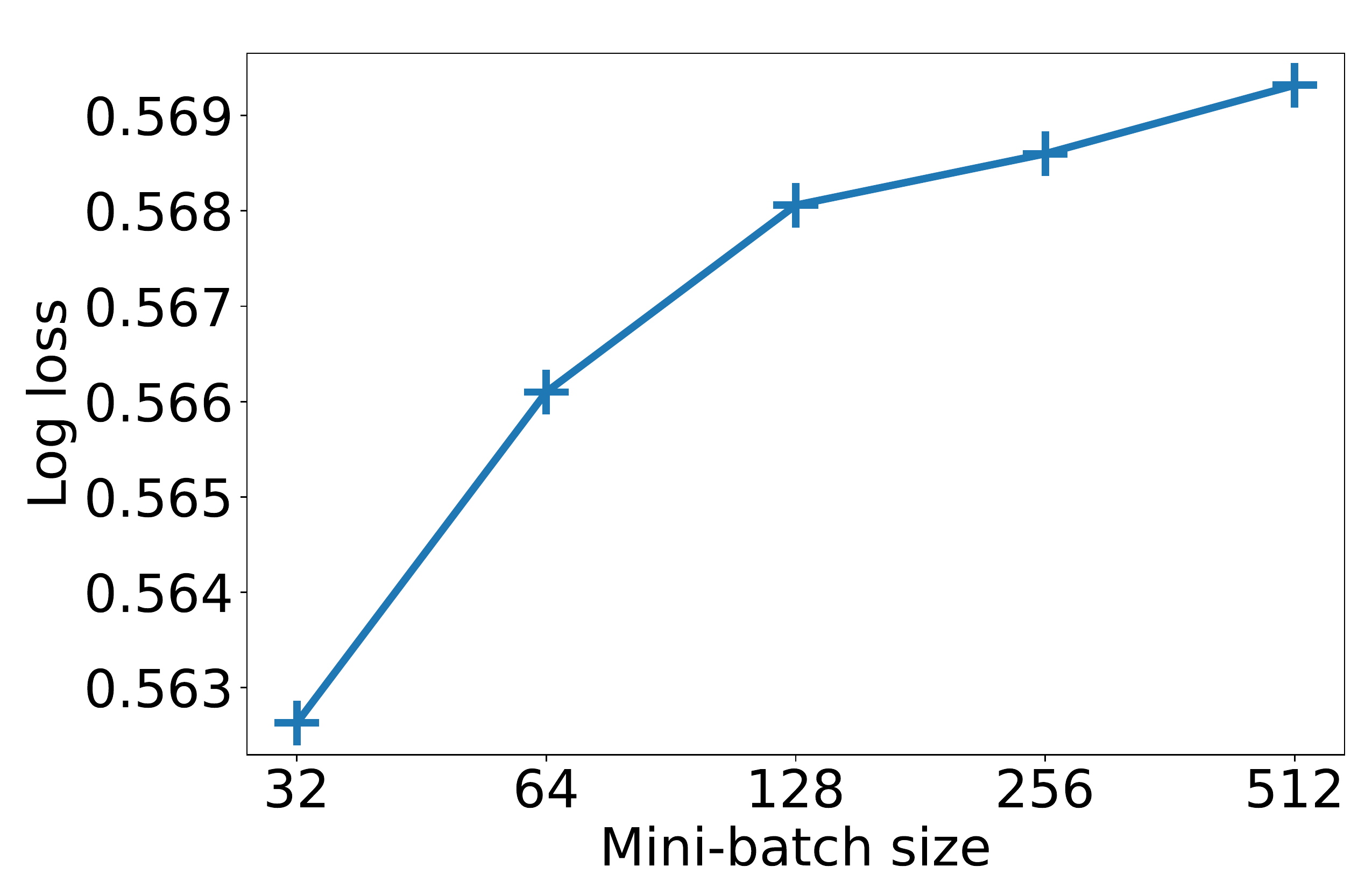}
         \caption{Test error increases with larger mini-batch size under random bags\label{fig:dnn_random_bag}}
     \end{subfigure} 
    \caption{The test log loss of the generalized additive model using neural nets as sub-models (DNN) and the generalized additive model using linear feature crosses as sub-models (linear feature crosses) on the Criteo Small dataset.\label{fig:experiment_result}}
    \label{fig:my_label}
\end{figure*}

We assess the effectiveness of our proposed algorithm on the Criteo Ads dataset~\citep{criteodata}. The dataset contains 41 million records with 13 integer and 26 categorical features. We convert the integer features into categorical features, resulting in a total of 39 categorical features. The data from the first 6 days is used for training, and the data from the 7th day is divided randomly into validation and test sets of equal size.
We select a specific number of feature column pairs and use curated bags to form bags for the dataset. We introduce noise to ensure that the aggregate label is $\eps$-differentially private at various privacy budgets $\eps$. For each chosen pair of feature columns, a single aggregate label will be generated. Each training sample will have the same number of aggregate labels.
A generalized additive model with multiple sub-models is constructed. Each sub-model is a multi-layer neural network that outputs a logit. The overall logit of the model is calculated by summing the logits of all the sub-models. The training loss is the sum of the losses for each aggregate label. The gradients of each sub-model's training loss can only be applied to its own parameters.
We evaluate the effectiveness of generalized additive models by using neural nets as sub-models and compare the results to those obtained by using linear feature crosses as sub-models. We refer to the generalized additive model that utilizes neural nets as sub-models as DNN and the one that uses linear feature crosses as sub-models as linear feature crosses.

We first show that our proposed curated bags method can learn effectively from aggregate labels without any loss in performance. In \cref{fig:noise_free}, we compare the proposed generalized additive models with DNN sub-models using individual labels and aggregate labels. We also compare them with the same model structure with linear sub-models using aggregate labels. From \cref{fig:noise_free}, we can see that the performance curves of DNN on individual labels and aggregate labels are identical, which confirms that the quality of the generalized additive model on aggregate labels is the same as that of individual labels, confirming our theory. Additionally, the performance of the linear feature crosses model is inferior to that of DNN, highlighting the superiority of our proposed architecture under curated bags, and showing that it can outperform previous aggregate learning methods. This further demonstrates the great potential of our proposed method in aggregate learning.
We then assess the effect of different noise levels on the performance of the proposed model structures when applied to data that has been made $\eps$-label differentially private through the addition of noise. The results are presented in \cref{fig:sub_models_vs_eps}. We observe that as the noise level increases, the test errors also increase. Conversely, increasing the number of DNN sub-models results in a decrease in test errors.
In the final experiment, we evaluate the performance of using random bags for comparison. We train a DNN model using all the features and aggregate labels of randomly selected examples on each mini-batch, and vary the mini-batch sizes. The results, shown in \cref{fig:dnn_random_bag}, indicate that larger batch sizes result in higher test errors. Compared with curated bags, random bags have higher test errors consistently in the experiment. This experiment demonstrates that using curated bags is a more effective method than using random bags.

\section{Conclusion}\label{sec:conclusion}

This paper examines curated bags and random bags for generating  aggregate labels. We demonstrate that the learner can achieve lossless learning from aggregate labels when using curated bags. We also study the learnability problem of random bags. Our empirical study shows that our method for curated bags achieves lossless learning from aggregate labels, has a reasonable privacy-utility trade-off when using differential privacy noise, and outperforms random bags.

\bibliographystyle{plainnat}
\bibliography{learn_sample}

\newpage
\onecolumn
\appendix
\section{Societal Impact}\label{sec:societal}

Our work can have a positive societal impact as it enables solutions with good utility and privacy tradeoff in the context of aggregate learning. Our approach allows for the development of more accurate and useful models while also protecting the sensitive information of individuals.  This balance between data utility and privacy protection is crucial in today's data-driven society and can have wide-reaching benefits for various industries, such as healthcare and finance. Additionally, it can also foster trust in the use of data and machine learning among the general public. Overall, our work helps to promote responsible data use and can lead to a more equitable and just society. However, it is important to note that this approach also has limitations, such as label differential privacy, which only protects labels and not feature columns and can be misused.

\section{Proof of \cref{thm:lossless-general}}

\begin{proof}[Proof of \cref{thm:lossless-general}]
We compute the total loss on the training dataset
\begin{equation*}
    \begin{split}
   & \sum_{(x,y)\in \Str} \frac{\partial \ell(y,\hat{y})}{\partial \beta_{j_0}}\\ ={}& \sum_{(x,y)\in \Str}(\nabla_{\hat{y}} b(\hat{y}) - T(y))^\top\left( \sum_{j\in [n_E]:j_0\in E_j} \frac{\partial f(x_{E'_j};\beta_{E_j})}{\partial \beta_{j_0}}\right)\\
    ={}& \sum_{j\in [n_E]:j_0\in E_j} \sum_{(x,y)\in \Str}(\nabla_{\hat{y}} b(\hat{y}) - T(y))^\top\left(  \frac{\partial f(x_{E'_j};\beta_{E_j})}{\partial \beta_{j_0}}\right)\\
    ={}& \sum_{j\in [n_E]:j_0\in E_j} 
    \sum_{(X,\bar{y})\in \CB(C_{\phi(j)})}
    \sum_{x\in X}
    (\nabla_{\hat{y}} b(\hat{y}) - T(y))^\top\left(  \frac{\partial f(x_{E'_j};\beta_{E_j})}{\partial \beta_{j_0}}\right)\,.
\end{split}
\end{equation*}
In the last line of the above equation, $y$ and $\hat{y}$ in the summand are the true label and model prediction of $x\in X$, respectively. Recall that the feature value of the feature columns $C_{\phi(j)}$ is identical and equal to $E'_j(X)$ for every $x\in X$ due to the construction of curated bags. As a result, we can take $\frac{\partial f(x_{E'_j};\beta_{E_j})}{\partial \beta_{j_0}}$ outside of the innermost summation and obtain \begin{equation*}
    \begin{split}
        & \sum_{(x,y)\in \Str} \frac{\partial \ell(y,\hat{y})}{\partial \beta_{j_0}}\\ ={}& \sum_{j\in [n_E]:j_0\in E_j} 
    \sum_{(X,\bar{y})\in \CB(C_{\phi(j)})}
   \left(  \frac{\partial f(E'_j(X);\beta_{E_j})}{\partial \beta_{j_0}}\right)^\top \sum_{x\in X}
    (\nabla_{\hat{y}} b(\hat{y}) - T(y))\\
    ={}& \sum_{j\in [n_E]:j_0\in E_j} 
    \sum_{(X,\bar{y})\in \CB(C_{\phi(j)})}
   \left(  \frac{\partial f(E'_j(X);\beta_{E_j})}{\partial \beta_{j_0}}\right)^\top \sum_{x\in X}
    (\nabla_{\hat{y}} b(\hat{y}) - \bar{y})\,.
    \end{split}
\end{equation*}
The last line of the equation above uses the definition of the aggregate label, $\bar{y} = \frac{1}{|X|}\sum_x\in X T(y)$, where $y$ is the true label of $x$. We do not introduce notation to emphasize the dependence of $y$ on $x$ in order to avoid complicated notation. 
\end{proof}

\section{Proof of \cref{thm:rademacher}}\label{sec:proof-learnability}

We define two auxiliary risks that use mean squared error and we will use them throughout this section. The first one is the expected Euclidean distance between the average prediction of the hypothesis $h$ on the bag of instances $X$ and its corresponding aggregate label $\bar{y}$:

$$R_{1}(h) = \mathbb{E}_{X,\bar{y}\sim\ag(\ep)}\left[\left\Vert \frac{1}{m}\sum_{x\in X}h(x)-\bar{y}\right\Vert _{2}^{2}\right]$$

Given the bags $X_1,\dots,X_n$ and their aggregate labels $\bar{y}_1,\dots,\bar{y}_n$, the corresponding empirical risk can be written as follows:

$$\hat{R}_{1}(h) = \frac{1}{n}\sum_{i\in[n]}\left\Vert \frac{1}{m}\sum_{x\in X_{i}}h(x)-\bar{y}_{i}\right\Vert _{2}^{2}$$

\cref{lem:e_avr_wor} computes the expected value of the square of the norm of the average of vectors sampled \emph{without replacement} from a finite set. 

\begin{lem}
\label{lem:e_avr_wor}Let $\{x_{i}\mid i\in[m]\}$ be sampled uniformly
from a finite set $S$ ($|S|=N$) without replacement. Then we have
\[
\bE\left\Vert \frac{1}{m}\sum_{i\in[m]}x_{i}\right\Vert _{2}^{2}=\frac{\left(m-1\right)N\left\Vert \bE_{x\sim\unif(S)}x\right\Vert _{2}^{2}+\left(N-m\right)\bE_{x\sim\unif(S)}\left\Vert x\right\Vert _{2}^{2}}{m\left(N-1\right)}\,.
\]
\end{lem}
\begin{proof}
We have
\begin{align*}
 & \bE\left\Vert \frac{1}{m}\sum_{i\in[m]}x_{i}\right\Vert _{2}^{2}\\
= & \frac{1}{m^{2}}\bE\left(\sum_{i\in[m]}\left\Vert x_{i}\right\Vert _{2}^{2}+\sum_{i\in[m]}\sum_{j\in[m]\setminus\{i\}}x_{i}^{\t}x_{j}\right)\\
= & \frac{1}{m^{2}}\left(m\bE_{x\sim\unif\left(S\right)}\left\Vert x\right\Vert _{2}^{2}+\bE\sum_{i\in[m]}\sum_{j\in[m]\setminus\{i\}}x_{i}^{\t}x_{j}\right)\,.
\end{align*}
The expected sum of cross terms is given by
\begin{align*}
\bE\sum_{i\in[m]}\sum_{j\in[m]\setminus\{i\}}= & \frac{m\left(m-1\right)}{N\left(N-1\right)}\sum_{x\in S}\sum_{y\in S\setminus\{x\}}x^{\t}y=\frac{m\left(m-1\right)}{N\left(N-1\right)}\left(\left\Vert \sum_{x\in S}x\right\Vert _{2}^{2}-\sum_{x\in S}\left\Vert x\right\Vert _{2}^{2}\right)\\
= & \frac{m\left(m-1\right)}{N-1}\left(N\left\Vert \bE_{x\sim\unif(S)}x\right\Vert _{2}^{2}-\bE_{x\sim\unif(S)}\left\Vert x\right\Vert _{2}^{2}\right)\,.
\end{align*}
Therefore, 
\[
\bE\left\Vert \frac{1}{m}\sum_{i\in[m]}x_{i}\right\Vert _{2}^{2}=\frac{\left(m-1\right)N\left\Vert \bE_{x\sim\unif\left(S\right)}x\right\Vert _{2}^{2}+\left(N-m\right)\bE_{x\sim\unif\left(S\right)}\left\Vert x\right\Vert _{2}^{2}}{m\left(N-1\right)}\,.
\]
\end{proof}

\cref{lem:R1(h)} presents an expression for the alternative risk $R_1(h)$. 
\begin{lem}
\label{lem:R1(h)}We have
\[
\begin{split}
    R_{1}(h)={}& \frac{\left(m-1\right)N\left\Vert \bE_{\left(x,y\right)\sim\ep}\left(h(x)-y\right)\right\Vert _{2}^{2}+\left(N-m\right)\bE_{\left(x,y\right)\sim\ep}\left\Vert h(x)-y\right\Vert _{2}^{2}}{m\left(N-1\right)}\\
    & +\left(1-\frac{1}{m}\right)\var_{y\sim\ep\mid_{y}}\left\Vert y\right\Vert _{2}^{2}\,.
\end{split}
\]
\end{lem}
\begin{proof}
We have

\begin{align}
R_{1}(h) & =\bE_{X,\bar{y}\sim\ag(\ep)}\left[\left\Vert \frac{1}{m}\sum_{j\in[m]}h(x_{j})-\frac{1}{m}\sum_{j\in[m]}y_{j}+\frac{1}{m}\sum_{j\in[m]}y_{j}-\bar{y}\right\Vert _{2}^{2}\right]\nonumber \\
 & =\bE_{X,\bar{y}\sim\ag(\ep)}\left[\left\Vert \frac{1}{m}\sum_{j\in[m]}h(x_{j})-\frac{1}{m}\sum_{j\in[m]}y_{j}\right\Vert _{2}^{2}\right]+\bE_{(X,\bar{y})\sim\ag(\ep)}\left[\left\Vert \frac{1}{m}\sum_{j\in[m]}y_{j}-\bar{y}\right\Vert _{2}^{2}\right]\nonumber \\
 & =\bE_{X,\bar{y}\sim\ag(\ep)}\left[\left\Vert \frac{1}{m}\sum_{j\in[m]}\left(h(x_{j})-y_{j}\right)\right\Vert _{2}^{2}\right]+\bE_{(X,\bar{y})\sim\ag(\ep)}\left[\left\Vert \frac{1}{m}\sum_{j\in[m]}y_{j}-\bar{y}\right\Vert _{2}^{2}\right]\,.\label{eq:2-terms}
\end{align}
First, we compute the second term in \eqref{eq:2-terms}
\begin{align*}
 & \bE_{(X,\bar{y})\sim\ag(\ep)}\left[\left\Vert \frac{1}{m}\sum_{j\in[m]}y_{j}-\bar{y}\right\Vert _{2}^{2}\right]\\
= & \bE_{(X,\bar{y})\sim\ag(\ep)}\sum_{k\in[K]}\left(\frac{1}{m}\sum_{j\in[m]}y_{j}[k]-\bar{y}[k]\right)^{2}\\
= & \sum_{k\in[K]}\bE\left[\bE\left[\left(\frac{1}{m}\sum_{j\in[m]}y_{j}[k]-\bar{y}[k]\right)^{2}\mid\frac{1}{m}\sum_{j\in[m]}y_{j}[k]\right]\right]\\
= & \sum_{k\in[K]}\bE\left[\left(\frac{1}{m}\sum_{j\in[m]}y_{j}[k]\right)\left(1-\frac{1}{m}\sum_{j\in[m]}y_{j}[k]\right)\right]\\
= & \sum_{k\in[K]}\left(\bE_{y\sim\ep\mid_{y}}\left[y[k]\right]-\left(\frac{1}{m}\bE_{y\sim\ep\mid_{y}}\left[y[k]^{2}\right]+\left(1-\frac{1}{m}\right)\left[\bE_{y\sim\ep\mid_{y}}y[k]\right]^{2}\right)\right)\\
= & \left(1-\frac{1}{m}\right)\var_{y\sim\ep\mid_{y}}\left\Vert y\right\Vert _{2}^{2}
\end{align*}
In the sequel, we  compute the first term in \eqref{eq:2-terms}. By
\cref{lem:e_avr_wor}, we have
\begin{align*}
 & \bE_{X,\bar{y}\sim\ag(\ep)}\left[\left\Vert \frac{1}{m}\sum_{j\in[m]}\left(h(x_{j})-y_{j}\right)\right\Vert _{2}^{2}\right]\\
= & \frac{\left(m-1\right)N\left\Vert \bE_{\left(x,y\right)\sim\ep}\left(h(x)-y\right)\right\Vert _{2}^{2}+\left(N-m\right)\bE_{\left(x,y\right)\sim\ep}\left\Vert h(x)-y\right\Vert _{2}^{2}}{m\left(N-1\right)}\,.
\end{align*}
Putting them together yields
\[
\begin{split}
    R_{1}(h)={}& \frac{\left(m-1\right)N\left\Vert \bE_{\left(x,y\right)\sim\ep}\left(h(x)-y\right)\right\Vert _{2}^{2}+\left(N-m\right)\bE_{\left(x,y\right)\sim\ep}\left\Vert h(x)-y\right\Vert _{2}^{2}}{m\left(N-1\right)}\\
    & +\left(1-\frac{1}{m}\right)\var_{y\sim\ep\mid_{y}}\left\Vert y\right\Vert _{2}^{2}\,.
\end{split}
\]
\end{proof}
\begin{lem}
\label{lem:R-infR}Define $\hat{r}(h)\triangleq\left\Vert \frac{1}{n}\sum_{i\in[n]}\left(\frac{1}{m}\sum_{x\in X_{i}}h(x)-\bar{y}_{i}\right)\right\Vert _{2}^{2}$
and
\begin{align}
\Delta_{1}(h) & \triangleq\bE_{(x,y)\sim\ep}\left\Vert h(x)-y\right\Vert _{2}^{2}-R(h)\\
& =\bE_{(x,y)\sim\ep}\left\Vert h(x)-y\right\Vert _{2}^{2}-\bE_{(x,y)\sim\cD}\left[\left\Vert h(x)-y\right\Vert _{2}^{2}\right]\in\bR\,,\label{eq:Delta_1}\\
\Delta_{3}(h) & \triangleq\left\Vert \bE_{(x,y)\sim\ep}\left(h(x)-y\right)\right\Vert _{2}^{2}-\left\Vert \frac{1}{n}\sum_{i\in[n]}\left(\frac{1}{m}\sum_{x\in X_{i}}h(x)-\bar{y}_{i}\right)\right\Vert _{2}^{2}\\
& =\left\Vert \bE_{(x,y)\sim\ep}\left(h(x)-y\right)\right\Vert _{2}^{2}-\hat{r}(h)\,.\label{eq:Delta_3}
\end{align}
If $\hat{h}\in\argmin_{h\in\cH}\left(\hat{R}_{1}(h)-\frac{\left(m-1\right)N}{m\left(N-1\right)}\hat{r}(h)\right)$,
we have
\[
R(\hat{h})-\inf_{h\in\cH}R(h)\le\frac{2\left(m-1\right)N}{N-m}\sup_{h\in\cH}\left\vert \Delta_{3}(h)\right\vert +2\sup_{h\in\cH}\left\vert \Delta_{1}(h)\right\vert +\frac{2m\left(N-1\right)}{N-m}\sup_{h\in\cH}\left\vert R_{1}(h)-\hat{R}_{1}(h)\right\vert \,.
\]
\end{lem}
\begin{proof}
By \cref{lem:R1(h)}, we have
\[
\begin{split}
    R_{1}(h)={}& \frac{\left(m-1\right)N\left\Vert \bE_{\left(x,y\right)\sim\ep}\left(h(x)-y\right)\right\Vert _{2}^{2}+\left(N-m\right)\left(\Delta_{1}(h)+R(h)\right)}{m\left(N-1\right)}\\
    & +\left(1-\frac{1}{m}\right)\var_{y\sim\ep\mid_{y}}\left\Vert y\right\Vert _{2}^{2}\,,
\end{split}
\]
which gives
\begin{equation}
\begin{split}
    R(h)-R_{1}(h)={}& \frac{\left(m-1\right)N}{m\left(N-1\right)}\left(R(h)-\left\Vert \bE_{(x,y)\sim\ep}\left(h(x)-y\right)\right\Vert _{2}^{2}\right)-\frac{N-m}{m\left(N-1\right)}\Delta_{1}(h)\\
    & -\left(1-\frac{1}{m}\right)\var_{y\sim\ep\mid_{y}}\left\Vert y\right\Vert _{2}^{2}\,.
\end{split}
\label{eq:R-R1}
\end{equation}

By \eqref{eq:R-R1}, for two hypotheses $\hat{h}$ and $h_{\eps}$,
we have

\begin{align}
 & \left(R(\hat{h})-R_{1}(\hat{h})\right)+\left(R_{1}(h_{\eps})-R(h_{\eps})\right)\nonumber \\
={} & \left(R(\hat{h})-R_{1}(\hat{h})\right)-\left(R(h_{\eps})-R_{1}(h_{\eps})\right)\nonumber \\
={} & \frac{\left(m-1\right)N}{m\left(N-1\right)}\left(R(\hat{h})-R(h_{\eps})+\left\Vert \bE_{(x,y)\sim\ep}\left(h_{\eps}(x)-y\right)\right\Vert _{2}^{2}-\left\Vert \bE_{(x,y)\sim\ep}\left(\hat{h}(x)-y\right)\right\Vert _{2}^{2}\right)\nonumber\\
& -\frac{N-m}{m\left(N-1\right)}\left(\Delta_{1}(\hat{h})-\Delta_{1}(h_{\eps})\right)\,.\label{eq:R-R1+R1-R}
\end{align}
Then we are in a position to compute and decompose $R(\hat{h})-R(h_{\eps})$
as follows
\begin{equation}
R(\hat{h})-R(h_{\eps})=\left(R(\hat{h})-R_{1}(\hat{h})\right)+\left(R_{1}(h_{\eps})-R(h_{\eps})\right)+\left(R_{1}(\hat{h})-R_{1}(h_{\eps})\right)\,.\label{eq:Rhhat-Rheps}
\end{equation}
Combining \eqref{eq:R-R1+R1-R} and \eqref{eq:Rhhat-Rheps}, we get
\begin{align*}
R(\hat{h})-R(h_{\eps})= & \frac{\left(m-1\right)N}{m\left(N-1\right)}\left(R(\hat{h})-R(h_{\eps})+\left\Vert \bE_{(x,y)\sim\ep}\left(h_{\eps}(x)-y\right)\right\Vert _{2}^{2}-\left\Vert \bE_{(x,y)\sim\ep}\left(\hat{h}(x)-y\right)\right\Vert _{2}^{2}\right)\\
 & -\frac{N-m}{m\left(N-1\right)}\left(\Delta_{1}(\hat{h})-\Delta_{1}(h_{\eps})\right)+\left(R_{1}(\hat{h})-R_{1}(h_{\eps})\right)\,.
\end{align*}
Re-arranging the terms gives
\begin{align*}
\frac{N-m}{m\left(N-1\right)}\left(R(\hat{h})-R(h_{\eps})\right)= & \frac{\left(m-1\right)N}{m\left(N-1\right)}\left(\left\Vert \bE_{(x,y)\sim\ep}\left(h_{\eps}(x)-y\right)\right\Vert _{2}^{2}-\left\Vert \bE_{(x,y)\sim\ep}\left(\hat{h}(x)-y\right)\right\Vert _{2}^{2}\right)\\
 & -\frac{N-m}{m\left(N-1\right)}\left(\Delta_{1}(\hat{h})-\Delta_{1}(h_{\eps})\right)+\left(R_{1}(\hat{h})-R_{1}(h_{\eps})\right)\,.
\end{align*}
Therefore, we haave
\begin{align}
R(\hat{h})-R(h_{\eps})= & \frac{\left(m-1\right)N}{N-m}\left(\left\Vert \bE_{(x,y)\sim\ep}\left(h_{\eps}(x)-y\right)\right\Vert _{2}^{2}-\left\Vert \bE_{(x,y)\sim\ep}\left(\hat{h}(x)-y\right)\right\Vert _{2}^{2}\right)\label{eq:Rhhat-Rheps2}\\
 & -\left(\Delta_{1}(\hat{h})-\Delta_{1}(h_{\eps})\right)+\frac{m\left(N-1\right)}{N-m}\left(R_{1}(\hat{h})-R_{1}(h_{\eps})\right)\,.
\end{align}

Define $\hat{r}(h)\triangleq\left\Vert \frac{1}{n}\sum_{i\in[n]}\left(\frac{1}{m}\sum_{x\in X_{i}}h(x)-\bar{y}_{i}\right)\right\Vert _{2}^{2}$
. Thus $\left\Vert \bE_{(x,y)\sim\ep}\left(h(x)-y\right)\right\Vert _{2}^{2}=\Delta_{3}(h)+\hat{r}(h)$.
We have
\begin{align}
 & \left\Vert \bE_{(x,y)\sim\ep}\left(h_{\eps}(x)-y\right)\right\Vert _{2}^{2}-\left\Vert \bE_{(x,y)\sim\ep}\left(\hat{h}(x)-y\right)\right\Vert _{2}^{2}\nonumber \\
= & \left(\Delta_{3}(h_{\eps})+\hat{r}(h_{\eps})\right)-\left(\Delta_{3}(\hat{h})+\hat{r}(\hat{h})\right)\nonumber \\
= & \left(\Delta_{3}(h_{\eps})-\Delta_{3}(\hat{h})\right)+\left(\hat{r}(h_{\eps})-\hat{r}(\hat{h})\right)\label{eq:D-D-norm-norm}
\end{align}
The term $R_{1}(\hat{h})-R_{1}(h_{\eps})$ can be decomposed into
three terms
\begin{equation}
R_{1}(\hat{h})-R_{1}(h_{\eps})=\left(R_{1}(\hat{h})-\hat{R}_{1}(\hat{h})\right)+\left(\hat{R}_{1}(\hat{h})-\hat{R}_{1}(h_{\eps})\right)+\left(\hat{R}_{1}(h_{\eps})-R_{1}(h_{\eps})\right)\,.\label{eq:3R}
\end{equation}
Combining \eqref{eq:Rhhat-Rheps2}, \eqref{eq:D-D-norm-norm} and \eqref{eq:3R}
gives
\begin{align*}
R(\hat{h})-R(h_{\eps})= & \frac{\left(m-1\right)N}{N-m}\left(\left(\Delta_{3}(h_{\eps})-\Delta_{3}(\hat{h})\right)+\left(\hat{r}(h_{\eps})-\hat{r}(\hat{h})\right)\right)-\left(\Delta_{1}(\hat{h})-\Delta_{1}(h_{\eps})\right)\\
 & +\frac{m\left(N-1\right)}{N-m}\left(\left(R_{1}(\hat{h})-\hat{R}_{1}(\hat{h})\right)+\left(\hat{R}_{1}(\hat{h})-\hat{R}_{1}(h_{\eps})\right)+\left(\hat{R}_{1}(h_{\eps})-R_{1}(h_{\eps})\right)\right)\,.
\end{align*}
Re-arranging the terms, we have
\begin{align*}
R(\hat{h})-R(h_{\eps})= & \frac{\left(m-1\right)N}{N-m}\left(\Delta_{3}(h_{\eps})-\Delta_{3}(\hat{h})\right)-\left(\Delta_{1}(\hat{h})-\Delta_{1}(h_{\eps})\right)\\
 & +\frac{m\left(N-1\right)}{N-m}\left(\left(R_{1}(\hat{h})-\hat{R}_{1}(\hat{h})\right)+\left(\hat{R}_{1}(h_{\eps})-R_{1}(h_{\eps})\right)\right)\\
 & +\frac{m\left(N-1\right)}{N-m}\left(\left(\hat{R}_{1}(\hat{h})-\frac{\left(m-1\right)N}{m\left(N-1\right)}\hat{r}(\hat{h})\right)-\left(\hat{R}_{1}(h_{\eps})-\frac{\left(m-1\right)N}{m\left(N-1\right)}\hat{r}(h_{\eps})\right)\right)\,.
\end{align*}
Since $\hat{h}\in\argmin_{h\in\cH}\left(\hat{R}_{1}(h)-\frac{\left(m-1\right)N}{m\left(N-1\right)}\left\Vert \frac{1}{n}\sum_{i\in[n]}\left(\frac{1}{m}\sum_{x\in X_{i}}h(x)-\bar{y}_{i}\right)\right\Vert _{2}^{2}\right)=\argmin_{h\in\cH}\left(\hat{R}_{1}(h)-\frac{\left(m-1\right)N}{m\left(N-1\right)}\hat{r}(h)\right)$,
we have $\left(\hat{R}_{1}(\hat{h})-\frac{\left(m-1\right)N}{m\left(N-1\right)}\hat{r}(\hat{h})\right)-\left(\hat{R}_{1}(h_{\eps})-\frac{\left(m-1\right)N}{m\left(N-1\right)}\hat{r}(h_{\eps})\right)\le0$.
Therefore, we get 
\begin{align*}
R(\hat{h})-R(h_{\eps})\le & \frac{\left(m-1\right)N}{N-m}\left(\Delta_{3}(h_{\eps})-\Delta_{3}(\hat{h})\right)-\left(\Delta_{1}(\hat{h})-\Delta_{1}(h_{\eps})\right)\\
 & +\frac{m\left(N-1\right)}{N-m}\left(\left(R_{1}(\hat{h})-\hat{R}_{1}(\hat{h})\right)+\left(\hat{R}_{1}(h_{\eps})-R_{1}(h_{\eps})\right)\right)\\
\le & \frac{2\left(m-1\right)N}{N-m}\sup_{h\in\cH}\left\vert \Delta_{3}(h)\right\vert +2\sup_{h\in\cH}\left\vert \Delta_{1}(h)\right\vert +\frac{2m\left(N-1\right)}{N-m}\sup_{h\in\cH}\left\vert R_{1}(h)-\hat{R}_{1}(h)\right\vert \,.
\end{align*}
Note that the above inequality holds for any $h_{\eps}$. For any
$\eps>0$, we pick $h_{\eps}$ such that $R(h_{\eps})\le\inf_{h\in\cH}R(h)+\eps$.
We have
\[
\begin{split}
    & R(\hat{h})-\inf_{h\in\cH}R(h)\\
    \le{}&  R(\hat{h})-R(h_{\eps})+\eps\\
    \le{}& \frac{2\left(m-1\right)N}{N-m}\sup_{h\in\cH}\left\vert \Delta_{3}(h)\right\vert +2\sup_{h\in\cH}\left\vert \Delta_{1}(h)\right\vert +\frac{2m\left(N-1\right)}{N-m}\sup_{h\in\cH}\left\vert R_{1}(h)-\hat{R}_{1}(h)\right\vert +\eps\,.
\end{split}
\]
Thus we conclude
\[
R(\hat{h})-\inf_{h\in\cH}R(h)\le\frac{2\left(m-1\right)N}{N-m}\sup_{h\in\cH}\left\vert \Delta_{3}(h)\right\vert +2\sup_{h\in\cH}\left\vert \Delta_{1}(h)\right\vert +\frac{2m\left(N-1\right)}{N-m}\sup_{h\in\cH}\left\vert R_{1}(h)-\hat{R}_{1}(h)\right\vert \,.
\]

\end{proof}
\begin{lem}[Adapted from Corollary 4 in \citep{maurer2016vector}]
\label{lem:vector_Rade}Let $\cX$ be any set, $\left(x_{1},\dots,x_{n}\right)\in\cX^{n}$,
let $\cF$ be a class of functions $f:\cX\to\bR^{K}$ and let $h_{i}:\bR^{K}\to\bR$
have Lipschitz norm $L$. Then
\[
\bE\sup_{f\in\cF}\sum_{i\in[n]}\eps_{i}h_{i}(f(x_{i}))\le\sqrt{2}L\bE\sup_{f\in\cF}\sum_{i\in[n],k\in[K]}\eps_{i,k}f(x_{i})[k]\,,
\]
where $\{\eps_{i}\mid i\in[n]\}$ and $\{\eps_{i,k}\mid i\in[n],k\in[K]\}$
are independent Rademacher random variables and $f(x_{i})[k]$ is
the $k$-th component of $f(x_{i})$. 
\end{lem}
\begin{lem}
\label{lem:sup_R1-R1hat}Define $\cG_{3}\triangleq\left\{ \left(X,\bar{y}\right)\mapsto\left\Vert \frac{1}{m}\sum_{x\in X}h(x)-\bar{y}\right\Vert _{2}^{2}\mid h\in\cH\right\} $.
If $R_{1}(h)\triangleq\bE_{X,\bar{y}\sim\ag(\ep)}\left[\left\Vert \frac{1}{m}\sum_{x\in X}h(x)-\bar{y}\right\Vert _{2}^{2}\right]$
and $\hat{R}_{1}(h)\triangleq\frac{1}{n}\sum_{i\in[n]}\left\Vert \frac{1}{m}\sum_{x\in X_{i}}h(x)-\bar{y}_{i}\right\Vert _{2}^{2}$,
with probability at least $1-\delta$, we have 
\[
\sup_{h\in\cH}\left\vert R_{1}(h)-\hat{R}_{1}(h)\right\vert \le4\sqrt{2K}\fR_{n,\ep\mid_{x}}(\cH)+\sqrt{\frac{\log\left(2/\delta\right)}{2n}}\,.
\]
\end{lem}
\begin{proof}
 Recall $\fR_{n,\ag(\ep)}(\cG_{3})=\bE_{\{(X_{i},\bar{y}_{i})\}\sim\ag(\ep)}\bE_{\{\sigma_{i}\}}\sup_{g\in\cG_{3}}\frac{1}{n}\sum_{i\in[n]}\sigma_{i}g\left(X_{i},\bar{y}_{i}\right)$.
By \citep{bartlett2002rademacher}, we have for any $\delta>0$, with
probability at least $1-\delta$
\begin{align}
 & \sup_{h\in\cH}\left\vert R_{1}(h)-\hat{R}_{1}(h)\right\vert \nonumber \\
= & \left|\bE_{X,\bar{y}\sim\ag(\ep)}\left[\left\Vert \frac{1}{m}\sum_{x\in X}h(x)-\bar{y}\right\Vert _{2}^{2}\right]-\frac{1}{n}\sum_{i\in[n]}\left\Vert \frac{1}{m}\sum_{x\in X_{i}}h(x)-\bar{y}_{i}\right\Vert _{2}^{2}\right|\nonumber \\
\le & 2\fR_{n,\ag(\ep)}(\cG_{3})+\sqrt{\frac{\log\left(2/\delta\right)}{2n}}\,.\label{eq:sup_R1h_hatR1h}
\end{align}
Since the function $\bR^{K}\ni y\mapsto\left\Vert y-\bar{y}\right\Vert _{2}^{2}$
is $2\sqrt{K}$-Lipschitz for $y,\bar{y}\in[0,1]$, by \cref{lem:vector_Rade},
we have
\begin{align}
& \fR_{n,\ag(\ep)}(\cG_{3})\\
 \le{}& 2\sqrt{2K}\bE_{\{X_{i}\}\sim\ag(\ep)\mid_{X}}\bE_{\{\sigma_{i,k}\}\iid\unif(\{\pm1\})}\sup_{h\in\cH}\frac{1}{n}\sum_{i\in[n],k\in[K]}\sigma_{i,k}\frac{1}{m}\sum_{j\in[m]}h(x_{i,j})[k]\nonumber \\
 \le{}& 2\sqrt{2K}\cdot\frac{1}{m}\sum_{j\in[m]}\bE_{\{X_{i}\}\sim\ag(\ep)\mid_{X}}\bE_{\{\sigma_{i,k}\}\iid\unif(\{\pm1\})}\sup_{h\in\cH}\frac{1}{n}\sum_{i\in[n],k\in[K]}\sigma_{i,k}h(x_{i,j})[k]\nonumber \\
 ={}& 2\sqrt{2K}\cdot\frac{1}{m}\sum_{j\in[m]}\bE_{\{x_{i}\}\sim\ep\mid_{x}}\bE_{\{\sigma_{i,k}\}\iid\unif(\{\pm1\})}\sup_{h\in\cH}\frac{1}{n}\sum_{i\in[n],k\in[K]}\sigma_{i,k}h(x_{i})[k]\nonumber \\
 ={}& 2\sqrt{2K}\bE_{\{x_{i}\}\sim\ep\mid_{x}}\bE_{\{\sigma_{i,k}\}\iid\unif(\{\pm1\})}\sup_{h\in\cH}\frac{1}{n}\sum_{i\in[n],k\in[K]}\sigma_{i,k}h(x_{i})[k]\nonumber \\
 ={}& 2\sqrt{2K}\fR_{n,\ep\mid_{x}}(\cH)\,.\label{eq:R_n_PStrain}
\end{align}
Combining \eqref{eq:sup_R1h_hatR1h} and \eqref{eq:R_n_PStrain} yields
the desired result.
\end{proof}
\begin{lem}
\label{lem:sup_Delta_1}If $\Delta_{1}(h)\triangleq\bE_{(x,y)\sim\ep}\left\Vert h(x)-y\right\Vert _{2}^{2}-\bE_{(x,y)\sim\cD}\left[\left\Vert h(x)-y\right\Vert _{2}^{2}\right]$,
with probability at least $1-\delta$, we have
\[
\sup_{h\in\cH}\left\vert \Delta_{1}(h)\right\vert \le4\sqrt{2K}\fR_{N,\bP\mid_{x}}^{+}(\cH)+\sqrt{\frac{\log2/\delta}{2N}}\,.
\]
\end{lem}
\begin{proof}
Define $\cG_{1}\triangleq\{\left(x,y\right)\mapsto\left\Vert h(x)-y\right\Vert _{2}^{2}\mid h\in\cH\}$
and recall 
\[
\fR_{N,\bP}(\cG_{1})\triangleq\bE_{\{(x_{i},y_{i})\mid i\in[N]\}\sim\bP}\bE_{\{\sigma_{i}\}\iid\unif(\{\pm1\})}\sup_{g\in\cG_{1}}\frac{1}{N}\sum_{i\in[N]}\sigma_{i}g(x_{i},y_{i})\,.
\]
By \citep{bartlett2002rademacher}, we have for any $\delta>0$, with
probability at least $1-\delta$
\begin{equation}
\sup_{h\in\cH}\left\vert \Delta_{1}(h)\right\vert \le2\fR_{N,\bP}(\cG_{1})+\sqrt{\frac{\log2/\delta}{2N}}\,.\label{eq:sup_Delta_1}
\end{equation}
Since the function $\bR^{K}\ni y\mapsto\left\Vert y-y'\right\Vert _{2}^{2}$
is $2\sqrt{K}$-Lipschitz for $y,y'\in[0,1]$, by \cref{lem:vector_Rade},
we have
\begin{equation}
\begin{split}
\fR_{N,\bP}(\cG_{1})\le{}& 2\sqrt{2K}\bE_{\{(x_{i},y_{i})\mid i\in[N]\}\sim\bP}\bE_{\{\sigma_{i,k}\}\iid\unif(\{\pm1\})}\sup_{h\in\cH}\frac{1}{N}\sum_{i\in[N],k\in[K]}\sigma_{i,k}h(x_{i})[k]\\
={}& 2\sqrt{2K}\fR_{N,\bP\mid_{x}}^{+}(\cH)\,.    
\end{split}
\label{eq:R_ND}
\end{equation}
Combining \eqref{eq:sup_Delta_1} and \eqref{eq:R_ND} yields the desired
result. 
\end{proof}
\begin{lem}
\label{lem:sup_Delta_3}Define $\cH\mid_{k}\triangleq\left\{ x\mapsto h(x)[k]\mid h\in\cH\right\} $.
If $\Delta_{3}(h)\triangleq\left\Vert \bE_{(x,y)\sim\ep}\left(h(x)-y\right)\right\Vert _{2}^{2}-\left\Vert \frac{1}{n}\sum_{i\in[n]}\left(\frac{1}{m}\sum_{x\in X_{i}}h(x)-\bar{y}_{i}\right)\right\Vert _{2}^{2}$,
with probability at least $1-2\delta$, we have
\[
\sup_{h\in\cH}\left|\Delta_{3}(h)\right|\le4\sum_{k\in[K]}\fR_{n,\ep}(\cH\mid_{k})+4K\sqrt{\frac{\log\left(2mK/\delta\right)}{2n}}\,.
\]
\end{lem}
\begin{proof}
We introduce three short-hand notations 
\begin{align*}
\rho_{1}(h) & \triangleq\bE_{(x,y)\sim\ep}\left(h(x)-y\right)\in\bR^{K}\,,\\
\hat{\rho}_{1,j}(h) & \triangleq\frac{1}{n}\sum_{i\in[n]}\left(h(x_{i,j})-y_{i,j}\right)\in\bR^{K}\,,\\
\rho_{2}(h) & \triangleq\frac{1}{n}\sum_{i\in[n]}\left(\frac{1}{m}\sum_{x\in X_{i}}h(x)-\bar{y}_{i}\right)\in\bR^{K}\,.
\end{align*}
 With these notations at hand, we have $\Delta_{3}(h)\triangleq\left\Vert \rho_{1}(h)\right\Vert _{2}^{2}-\left\Vert \rho_{2}(h)\right\Vert _{2}^{2}$.
Note that $\rho_{1}(h),\rho_{2}(h)\in[-1,1]^{K}$. We have
\begin{equation}
\left|\Delta_{3}(h)\right|\le\left|\sum_{k\in[K]}\left(\rho_{1}(h)[k]^{2}-\rho_{2}(h)[k]^{2}\right)\right|\le2\sum_{k\in[K]}\left|\rho_{1}(h)[k]-\rho_{2}(h)[k]\right|=2\left\Vert \rho_{1}(h)-\rho_{2}(h)\right\Vert _{1}\,.\label{eq:Delta_3-1}
\end{equation}
Define $\cG_{2}\mid_{k}\triangleq\{\left(x,y\right)\mapsto(h(x)-y)[k]\mid h\in\cH\}$.
By \citep{bartlett2002rademacher}, we have for any $\delta>0$, with
probability at least $1-\delta/(mK)$
\end{proof}
\begin{equation}
\sup_{h\in\cH}\left\vert \rho_{1}(h)[k]-\hat{\rho}_{1,j}(h)[k]\right\vert \le2\fR_{n,\ep}(\cG_{2}\mid_{k})+\sqrt{\frac{\log\left(2mK/\delta\right)}{2n}}\,.\label{eq:sup_rho1}
\end{equation}

Since 
\begin{align*}
\rho_{2}(h) & =\frac{1}{n}\sum_{i\in[n]}\left(\frac{1}{m}\sum_{j\in[m]}h(x_{i,j})-\bar{y}_{i}\right)\\
 & =\frac{1}{n}\sum_{i\in[n]}\left(\frac{1}{m}\sum_{j\in[m]}\left(h(x_{i,j})-y_{i,j}\right)+\frac{1}{m}\sum_{j\in[m]}y_{i,j}-\bar{y}_{i}\right)\\
 & =\frac{1}{m}\sum_{j\in[m]}\hat{\rho}_{1,j}(h)+\frac{1}{n}\sum_{i\in[n]}\left(\frac{1}{m}\sum_{j\in[m]}y_{i,j}-\bar{y}_{i}\right)\,,
\end{align*}
we have
\begin{align*}
 & \left\Vert \rho_{1}(h)-\rho_{2}(h)\right\Vert _{1}\\
= & \left\Vert \rho_{1}(h)-\left(\frac{1}{m}\sum_{j\in[m]}\hat{\rho}_{1,j}(h)+\frac{1}{n}\sum_{i\in[n]}\left(\frac{1}{m}\sum_{j\in[m]}y_{i,j}-\bar{y}_{i}\right)\right)\right\Vert _{1}\\
= & \left\Vert \rho_{1}(h)-\frac{1}{m}\sum_{j\in[m]}\hat{\rho}_{1,j}(h)\right\Vert _{1}+\left\Vert \frac{1}{n}\sum_{i\in[n]}\left(\frac{1}{m}\sum_{j\in[m]}y_{i,j}-\bar{y}_{i}\right)\right\Vert _{1}\\
\le & \frac{1}{m}\sum_{j\in[m]}\left\Vert \rho_{1}(h)-\hat{\rho}_{1,j}(h)\right\Vert _{1}+\left\Vert \frac{1}{n}\sum_{i\in[n]}\left(\frac{1}{m}\sum_{j\in[m]}y_{i,j}-\bar{y}_{i}\right)\right\Vert _{1}\,.
\end{align*}
By \eqref{eq:sup_rho1}, with probability $1-\delta$, we have
\[
\begin{split}
    \frac{1}{m}\sum_{j\in[m]}\left\Vert \rho_{1}(h)-\hat{\rho}_{1,j}(h)\right\Vert _{1}={}& \frac{1}{m}\sum_{j\in[m],k\in[K]}\left|\rho_{1}(h)[k]-\hat{\rho}_{1,j}(h)[k]\right|\\
    \le{}& 2\sum_{k\in[K]}\fR_{n,\ep}(\cG_{2}\mid_{k})+K\sqrt{\frac{\log\left(2mK/\delta\right)}{2n}}\,.
\end{split}
\]

Recall $\bar{y}_{i}\mid y_{i,1},y_{i,2},\dots,y_{i,m}\sim\ber\left(\frac{1}{m}\sum_{j\in[m]}y_{i,j}\right)$,
by Hoeffding's inequality, we get 
\[
\Pr\left(\left\vert \frac{1}{n}\sum_{i\in[n]}\left(\frac{1}{m}\sum_{j\in[m]}y_{i,j}-\bar{y}_{i}\right)[k]\right\vert \ge\sqrt{\frac{\log(2K/\delta)}{2n}}\right)\le\delta/K\,.
\]
With probability at least $1-\delta$, we have $\left\Vert \frac{1}{n}\sum_{i\in[n]}\left(\frac{1}{m}\sum_{j\in[m]}y_{i,j}-\bar{y}_{i}\right)\right\Vert _{1}\le K\sqrt{\frac{\log(2K/\delta)}{2n}}$.
Therefore, with probability at least $1-2\delta$, we have
\[
\left\Vert \rho_{1}(h)-\rho_{2}(h)\right\Vert _{1}\le2\sum_{k\in[K]}\fR_{n,\ep}(\cG_{2}\mid_{k})+2K\sqrt{\frac{\log\left(2mK/\delta\right)}{2n}}\,,
\]
which implies 
\begin{align*}
\sup_{h\in\cH}\left|\Delta_{3}(h)\right| & \le4\sum_{k\in[K]}\fR_{n,\ep}(\cG_{2}\mid_{k})+4K\sqrt{\frac{\log\left(2mK/\delta\right)}{2n}}\\
 & \le4\sum_{k\in[K]}\fR_{n,\ep\mid_{x}}(\cH\mid_{k})+4K\sqrt{\frac{\log\left(2mK/\delta\right)}{2n}}\,.
\end{align*}
The second inequality above is because of Talagrand's contraction
lemma (see, e.g., \citep[Lemma 5.7]{mohri2018foundations}).
\begin{proof}[Proof of \cref{thm:rademacher}]
Using \cref{lem:R-infR}, \cref{lem:sup_R1-R1hat}, \cref{lem:sup_Delta_1}
and \cref{lem:sup_Delta_3}, with probability at least $1-4\delta$,
we have
\begin{align*}
 & R(\hat{h})-\inf_{h\in\cH}R(h)\\
\le & \frac{2\left(m-1\right)N}{N-m}\sup_{h\in\cH}\left\vert \Delta_{3}(h)\right\vert +2\sup_{h\in\cH}\left\vert \Delta_{1}(h)\right\vert +\frac{2m\left(N-1\right)}{N-m}\sup_{h\in\cH}\left\vert R_{1}(h)-\hat{R}_{1}(h)\right\vert \\
\le & \frac{2\left(m-1\right)N}{N-m}\left(4\sum_{k\in[K]}\fR_{n,\ep\mid_{x}}(\cH\mid_{k})+4K\sqrt{\frac{\log\left(2mK/\delta\right)}{2n}}\right)\\
 & +2\left(4\sqrt{2K}\fR_{N,\cD_{x}}^{+}(\cH)+\sqrt{\frac{\log2/\delta}{2N}}\right)\\
 & +\frac{2m\left(N-1\right)}{N-m}\left(4\sqrt{2K}\fR_{n,\ep\mid_{x}}(\cH)+\sqrt{\frac{\log\frac{2}{\delta}}{2n}}\right)\,.
\end{align*}
\end{proof}

\end{document}

%% file: math_commands.tex
\usepackage{amsmath,amsfonts,bm}

\def\eqref#1{equation~\ref{#1}}

\def\1{\bm{1}}

\def\eps{{\epsilon}}

\DeclareMathAlphabet{\mathsfit}{\encodingdefault}{\sfdefault}{m}{sl}
\SetMathAlphabet{\mathsfit}{bold}{\encodingdefault}{\sfdefault}{bx}{n}

\DeclareMathOperator*{\argmin}{arg\,min}